\documentclass[10pt]{article} 
\usepackage[preprint]{rlc}

\renewcommand{\P}[1]{\mathbb{P}\left(#1 \right)} 
\newcommand{\E}[1]{\mathbb{E}\left[ #1 \right]} 
\newcommand{\N}{\mathbb{N}} 
\renewcommand{\S}{\mathbb{S}} 

\definecolor{warningcol}{rgb}{.99,.1,.5}
\definecolor{todocol}{rgb}{.4,.4,.8}
\definecolor{sketchcol}{rgb}{.4,.4,.8}
\definecolor{outlinecol}{rgb}{.8,.4,.3}

\renewcommand{\E}{\mathbb{E}}

\newcommand{\pbr}[1]{\left( #1\right)}
\newcommand{\sbr}[1]{\left[ #1\right]}
\newcommand{\cbr}[1]{\left\{ #1\right\}}

\makeatother

\renewcommand{\S}{\mathcal{S}}
\renewcommand{\P}{\mathbb{P}}

\usepackage{algorithm}
\usepackage{algpseudocode}
\usepackage{mathtools}
\usepackage[normalem]{ulem}

\newcommand\footnoteref[1]{\protected@xdef\@thefnmark{\ref{#1}}\@footnotemark}

\usepackage{booktabs}
\usepackage{amssymb}            
\usepackage{mathtools}          
\usepackage{amsthm}
\usepackage{mathrsfs}   

\mathtoolsset{showonlyrefs}     
\usepackage{graphicx}           
\usepackage{subfigure}         
\usepackage[space]{grffile}     
\usepackage{url}                
\usepackage{graphicx} 

\DeclarePairedDelimiter\abs{\lvert}{\rvert}
\DeclarePairedDelimiter\norm{\lVert}{\rVert}
\DeclarePairedDelimiter\innorm{\langle}{\rangle}

\DeclareMathOperator{\R}{\mathbb{R}} 

\DeclareMathOperator*{\argmin}{arg\,min}

\DeclareMathOperator{\St}{\mathcal{S}}
\DeclareMathOperator{\A}{\mathcal{A}}
\DeclareMathOperator{\X}{\mathcal{X}}

\newcommand{\Rc}{\mathcal{R}}

\newcommand\blfootnote[1]{%
  \begingroup
  \renewcommand\thefootnote{}\footnote{#1}%
  \addtocounter{footnote}{-1}%
  \endgroup
}
\usepackage{bbm}

\newtheorem{assumption}{Assumption}
\newtheorem{lemma}{Lemma}
\newtheorem{theorem}{Theorem}

\title{ On the Global Convergence of Policy Gradient in Average Reward Markov Decision Processes}

\author{Navdeep Kumar*  \\
    navdeepkumar@campus.technion.ac.il \\
    Electrical and Computer Engineering\\
    Technion - Israel Institute of Technology
    \And
    Yashaswini Murthy* \\
    ymurthy2@illinois.edu\\
    ECE \& CSL\\
    University of Illinois Urbana-Champaign
    \And
    Itai Shufaro \\
    itai.shufaro@campus.technion.ac.il\\
    Electrical and Computer Engineering\\
    Technion - Israel Institute of Technology
    \And
    Kfir Y. Levy \\
    kfirylevy@technion.ac.il\\
    Electrical and Computer Engineering\\
    Technion - Israel Institute of Technology
    \And
    R. Srikant \\
    rsrikant@illinois.edu\\
    ECE \& CSL\\
    University of Illinois Urbana-Champaign
    \And 
    \qquad \qquad \ \ 
    Shie Mannor \\
    \qquad \qquad \ \ \ \ shie@ee.technion.ac.il\\
    \ \qquad \qquad \ \ \ Electrical Engineering \\
    \ \qquad \qquad \ \ \ Technion - Israel Institute of Technology \\
    \ \qquad \qquad \ \ \ NVIDIA Research
    }

\usepackage[toc,page,header]{appendix}
\usepackage{minitoc}

\begin{document}
\doparttoc 
\faketableofcontents 

\part{} 

\maketitle

\begin{abstract}
    We present the first finite time global convergence analysis of policy gradient in the context of infinite horizon average reward Markov decision processes (MDPs). Specifically, we focus on ergodic tabular MDPs with finite state and action spaces. Our analysis shows that the policy gradient iterates converge to the optimal policy at a sublinear rate of $O\pbr{\frac{1}{T}},$ which translates to $O\pbr{\log(T)}$ regret, where $T$ represents the number of iterations. Prior work on performance bounds for discounted reward MDPs cannot be extended to average reward MDPs because the bounds grow proportional to the fifth power of the effective horizon. Thus, our primary contribution is in proving that the policy gradient algorithm converges for average-reward MDPs and in obtaining finite-time performance guarantees. In contrast to the existing discounted reward performance bounds, our performance bounds have an explicit dependence on constants that capture the complexity of the underlying MDP. Motivated by this observation, we reexamine and improve the existing performance bounds for discounted reward MDPs. We also present simulations to empirically evaluate the performance of average reward policy gradient algorithm. 
\end{abstract}

\section{Introduction}\blfootnote{*Equal contribution}
Average reward Markov Decision Processes (MDPs) find applications in various domains where decisions need to be made over time to optimize long-term performance. Some of these applications include resource allocation, portfolio management in finance, healthcare and robotics \citep{ghalme2021long,bielecki1999value,patrick2011markov,mahadevan1996average,tadepalli1998model}. Approaches for determining the optimal policy can be broadly categorized into dynamic programming algorithms (such as value and policy iteration\citep{murthy2024performance,abbasi2019politex,gosavi2004reinforcement}) and gradient-based algorithms. Although gradient based algorithms are heavily used in practice\citep{schulman2015trust,baxter2000direct}, the theoretical analysis of their global convergence is a relatively recent undertaking.

While extensive research has been conducted on the global convergence of policy gradient methods in the context of discounted reward MDPs\citep{agarwal2020theory,khodadadian2021linear}, comparatively less attention has been given to its average reward counterpart. Contrary to average reward MDPs, the presence of a discount factor ($\gamma<1$) serves as a source of contraction that alleviates the technical challenges involved in analyzing the performance of various algorithms in the context of discounted reward MDPs. Consequently, many algorithms designed for average reward MDPs are evaluated using the framework of discounted MDPs, where the discount factor approaches one\citep{grand2024reducing}. 

In the context of discounted reward MDPs, the projected policy gradient algorithm converges as follows:
\begin{equation}
    \rho^*_\gamma - \rho_\gamma^{\pi_k}\leq O\pbr{\frac{1}{\pbr{1-\gamma}^5}}
    \label{eq1}
\end{equation}
where $\gamma$ is the discount factor, $\rho^*_\gamma$ represents the optimal value function and $\rho_\gamma^{\pi_k}$ represents the value function iterates obtained through projected gradient ascent \citep{xiao2022convergence}. Let $\rho^\pi$ denote the average reward associated with some policy $\pi$. It is well known that $\rho^\pi=\lim_{\gamma\to 1}(1-\gamma)\rho_\gamma^\pi$ under some mild conditions\citep{Puterman1994MarkovDP,bertdimitri07book}. Utilizing this relationship in \eqref{eq1}, we observe the upper bound tends to infinity in the limit $\gamma\to 1$. Hence, it is necessary to devise an alternate approach to study the convergence of policy gradient in the context of average reward MDPs. 

\subsection{Related Work}

\citep{fazel2018global} were among the first to establish the global convergence of policy gradients, specifically within the domain of linear quadratic regulators. \citep{bhandari2024global} established a connection between the policy gradient and policy iteration objectives, determining conditions under which policy gradient algorithms converge to the globally optimal solution. \citep{agarwal2020theory} offer convergence bounds of $O(\frac{1}{(1-\gamma)^6})$ for policy gradient and $O(\frac{1}{(1-\gamma)^2})$ for natural policy gradient. It is noteworthy that while their convergence bounds for policy gradient rely on the cardinality of the state and action space, the convergence bounds for natural policy gradient are independent of them. \citep{xiao2022convergence} enhances the $O(\frac{1}{(1-\gamma)^6})$ policy gradient bounds by refining the dependency on the discount factor, yielding improved bounds of $O(\frac{1}{(1-\gamma)^5})$. \citep{zhang2020variational} prove that variance reduced versions of both policy gradient and natural policy gradient algorithms also converge to the global optimal solution. \citep{mei2020global} analyze global convergence of softmax based gradient methods and prove exponential rejection of suboptimal policies. However, all of these studies have focused on discounted reward MDPs.

In the context of average reward MDPs, local convergence properties of policy gradient have been extensively studied. \citep{AvgRewardKonda1999BellmanOperator} explore two time-scale actor-critic algorithms with projected gradient ascent updates for the actor. They employ linear function approximation to represent the value function and establish asymptotic local convergence of the average reward to a stationary point. \citep{bhatnagar2009natural} consider four different types of gradient based methods, including natural policy gradients and prove their asymptotic local convergence using an ODE-style analysis. The global convergence of natural policy gradient for average reward MDPs has been explored by \citep{even2009online} and \citep{murthy2023convergence} for finite state and action spaces, while \citep{grosof2024convergence} has addressed the case of infinite state space. The only work known to us addressing the global convergence of policy gradient in average reward MDPs is \citep{bai2023regret}. While they investigate parameterized policies in a learning scenario, their analysis hinges on the crucial assumption of the average reward being smooth with respect to the policy parameter. However, it was unclear as to what type of average reward MDPs satisfy such an assumption. Our primary contribution here is to prove the smoothness of the average reward holds for a very large class of MDPs. in the tabular setting. Additionally, their work offers a regret of $O(T^{1/4})$, while we achieve a sublinear regret of $O(\log(T))$.

\subsection{Contributions}
In this subsection, we outline the key contributions of this paper.
\begin{itemize}
    \item \textbf{Smoothness Analysis of Average Reward:} We introduce a new analysis technique to prove the smoothness of the average reward function with respect to the underlying policy. One of the key difficulties in proving such a result is the lack of uniqueness of the value function in average-reward problems, We overcome this difficulty by using a projection technique to ensure the uniqueness of the function and using the properties of the projection to prove smoothness.
    \item \textbf{Sublinear Convergence Bounds:} Using the above smoothness property, we present finite time bounds on the optimality gap over time, showing that the iterates approach the optimal policy with an overall regret of $O\pbr{\log{(T)}}.$ In contrast, the regret bounds in \citep{bai2023regret} are atleast $O\pbr{T^{\frac{1}{4}}}$ without learning error and with tabular parametrization. In place of the discount factor and the cardinality of the state and action spaces in the discounted setting, our finite-time performance bounds involves a different parameter which characterizes the complexity of the underlying MDP.
    \item \textbf{Extension to Discounted Reward MDPs:} Our analysis can also be applied to the discounted reward MDP problem to provide stronger results than the state of the art. In particular, we show that our performance bounds for discounted MDPs can be expressed in terms of a problem complexity parameter, which can be independent of the size of the state and action spaces in some problems.
    \item \textbf{Experimental Validation:} We simulate the performance of policy gradient across a simple class of MDPs to empirically evaluate its performance. The simulations show that, in many cases, our bounds are independent of the size of the state and action spaces.
\end{itemize}

\section{Preliminaries}
In this section, we introduce our model, address the limitations of applying the optimality gap bounds from discounted reward MDPs to the average reward scenario, present the gradient ascent update, and discuss the assumptions underlying our analysis. 

\subsection{Average Reward MDP Formulation}
We consider the class of infinite horizon average reward MDPs with finite state space $\S$ and finite action space $\A$. The environment is modeled as a probability transition kernel denoted by $\P$. We consider a class of randomized policies $\Pi=\{\pi:\S\to\Delta\A\}$, where a policy $\pi$ maps each state to a probability vector over the action space. The transition kernel corresponding to a policy $\pi$ is represented by $\P^\pi:\S\to\S$, where $\P^\pi(s'|s) = \sum_{a\in\A}\pi(a|s)\P(s'|s,a)$ denotes the single step probability of moving from state $s$ to $s'$ under policy $\pi.$ Let $r(s,a)$ denote the single step reward obtained by taking action $a\in\A$ in state $s\in\S$. The single-step reward associated with a policy $\pi$ at state $s\in\S$ is defined as $r^\pi(s)=\sum_{a\in\A}\pi(a|s)r(s,a).$

The infinite horizon average reward objective $\rho^\pi$ associated with a policy $\pi$ is defined as:
\begin{align}\label{def:averageReward}
\rho^\pi = \lim_{N\to\infty}\frac{\mathbb{E}_{\pi}\left[\sum_{n=0}^{N-1}{r^\pi(s_n)}\right]}{{N}},
\end{align}
where the expectation is taken with respect to $\P^\pi$. The average reward is independent of the initial state distribution under some mild conditions \citep{ros83book,bertdimitri07book} and can be alternatively expressed as $ \rho^\pi = \sum_{\substack{s\in\S}}d^\pi(s)r^\pi(s),$
where $d^\pi(s)$ is the stationary measure corresponding to state $s$ under the transition kernel $\P^\pi$, ensuring that $d^\pi$ satisfies the equation $d^\pi \P^\pi = d^\pi$. Associated with a policy is a relative state value function $v^\pi\in\R^{|\S|}$ that satisfies the following average reward Bellman equation 
\begin{align}\label{bellmaneq}
    \rho^\pi\mathbbm{1} + v^\pi = r^\pi + \P^\pi v^\pi,
\end{align}
where $\mathbbm{1}$ is the all ones vector \citep{Puterman1994MarkovDP,bertdimitri07book}. Note that $v^\pi$ is unique up to an additive constant. Setting $\sum_{s\in\S}d^\pi(s)v^\pi(s)=0$ imposes an additional constraint over $v^\pi$, providing a unique value function vector denoted by $v^\pi_0$, known as the basic differential reward function \citep{tsitsiklis1999average}. It can be shown that $v_0^\pi$ can alternatively expressed as $v^\pi_0(s)=\E_\pi\sbr{\sum_{n=0}^\infty\pbr{r^\pi(s_n)-\rho^\pi}|s_0=s}$. Hence any element in the set $\{{v}_0^\pi +c\mathbbm{1}: c\in\R\}$ is a solution $v^\pi$ to the Bellman equation \eqref{bellmaneq}.

The relative state action value function $Q^\pi\in\R^{\S\times\A}$ associated with a policy $\pi$ is defined as: 
\begin{align}\label{qfunction}
    Q^\pi(s,a) = r(s,a) + \sum_{\substack{s'\in\S\\a'\in\A}}\P\pbr{s'|s,a}\pi(a'|s') Q^\pi(s',a') - \rho^\pi \qquad \forall \pbr{s,a}\in\S\times\A
\end{align}
Similar to $v^\pi$, $Q^\pi$ is also unique up to an additive constant. Analogously, every solution $Q^\pi$ of \eqref{qfunction} can be expressed as an element in the set $\cbr{Q_0^\pi(s,a)+c\mathbbm{1}: c\in\R}$ where $Q_0^\pi(s,a) = \E_\pi\sbr{\sum_{n=0}^\infty\pbr{r^\pi(s_n)-\rho^\pi}|s_0=s,a_0=a}$. Upon averaging \eqref{qfunction} with policy $\pi$, it follows that $v^\pi(s)=\sum_{a\in\A}\pi(a|s)Q^\pi(s,a).$ 
The average reward policy gradient theorem \citep{Sutton1998} for policies parameterized by $\theta$ is given by:
\begin{equation}
    \frac{\partial\rho}{\partial\theta} = \sum_{s\in\S}d^\pi(s)\sum_{a\in\A}\frac{\partial\pi(s,a)}{\partial\theta}Q^\pi(s,a)
\end{equation}
As we focus on tabular policies in this paper, our parameterization aligns with the tabular policy, where $\theta$ is equivalent to $\pi.$ The policy gradient update considered is defined below. 
\begin{equation}\label{avg:pg}
     \pi_{k+1} :=\textbf{Proj}_{\Pi}\sbr{\pi_k +\eta\frac{\partial\rho^\pi}{\partial\pi}\Bigg|_{\pi=\pi_k}}\qquad \forall k\geq 0,  
\end{equation}
where $\textbf{Proj}_{\Pi}$ denotes the orthogonal projection in the Euclidean norm onto the space of randomized policies $\Pi$ and $\eta$ denotes the step size of the update. 
In the following subsection, we recall the policy gradient result within the framework of discounted reward MDPs and address why it cannot be directly applied to the average reward scenario.

\subsection{Relationship to discounted reward MDPs}
\textcolor{black}{Let $\rho^\pi_{\mu,\gamma}:= \mu^T(1-\gamma \P^\pi)r^\pi$} represent the discounted reward value function associated with policy $\pi$, under the initial distribution $\mu\in\Delta\S$ and where $\gamma$ represents the discount factor\citep{bertdimitri07book}. Consider the projected policy gradient update given below.
\begin{equation}\label{disc:pg}
    \pi_{k+1} :=\textbf{Proj}_{\Pi}\sbr{\pi_k +\eta\frac{\partial\rho_{\mu,\gamma}^\pi}{\partial\pi}\Bigg|_{\pi=\pi_k} }\qquad \forall k\geq 0.  
\end{equation}
When the step size $\eta=\frac{\pbr{1-\gamma}^3}{2\gamma|\A|}$, the iterates $\pi_k$ generated from projected gradient ascent \eqref{disc:pg} satisfy the following equation:
\begin{equation}\label{disc:perfbound}
   \rho^{*}_{\mu,\gamma} -  \rho^{\pi_k}_{\mu,\gamma} \leq \frac{128|\S||\A|}{k(1-\gamma)^5} \Bigg\|\frac{d_\mu(\pi^*)}{\mu}\Bigg\|_\infty^2,
\end{equation}
where $\rho^*_{\mu,\gamma}$ represents the optimal value function under initial distribution $\mu$, and \textcolor{black}{$d_{\mu,\gamma}^{\pi^*} :=(1-\gamma)\mu^T(1-\gamma \P^{\pi^*})^{-1} $ represents the state occupancy measure under optimal policy $\pi^*$.} Under some mild conditions the average reward $\rho^\pi$ associated with a policy $\pi$ and the value function $\rho^{\pi}_{\mu,\gamma}(s)$ are related as below:
\begin{equation}\label{avg-disc}
    \rho^\pi = \lim_{\gamma\to 1}(1-\gamma)\rho^{\pi}_{\mu,\gamma}(s).
\end{equation}
Note that the above relation\citep{bertdimitri07book,ros83book} holds for all $s\in\S$ and all $\mu\in\Delta\S$ since the average reward is independent of the initial state distribution. Upon leveraging the relation in \eqref{avg-disc} and multiplying \eqref{disc:perfbound} with $(1-\gamma)$, it is apparent that the upper bound of \eqref{disc:perfbound} in the limit of $\gamma\to 1$ tends to infinity. This is due to $\pbr{1-\gamma}^4$ that remains  in the denominator of $\eqref{disc:perfbound}$ upon multiplying with $(1-\gamma)$. Therefore it is necessary to devise an alternative proof technique in order to analyze the global convergence of policy gradient in the context of average reward MDPs. Prior to presenting the main result and its proof, we state the assumption used in our analysis.

\begin{assumption}
    For every policy $\pi\in\Pi,$ the transition matrix $\P^\pi$ associated with the induced Markov chain is irreducible and aperiodic. This assumption also means that there exist constants $C_e<\infty$ and $\lambda\in [0,1)$ such that for any $k\in\N$ and any $\pi\in\Pi$, the Markov chain corresponding to $\P^\pi$ is geometrically ergodic i.e., $\|\pbr{\P^\pi}^k-\mathbbm{1}\pbr{d^\pi}^\top\|_\infty\leq C_e\lambda^k.$
\end{assumption}

For simplicity, we will also assume $|r(s,a)|\in\sbr{0,1} \forall (s,a)\in(\S\times\A).$
We now present the main result of the paper. 

\section{Main Results}
\begin{theorem}\label{maintheorem}
    Let $\rho^{\pi_k}$ be the average reward corresponding to the policy iterates $\pi_k$, obtained through the policy gradient update \eqref{avg:pg}. Let $\rho^*$ represent the optimal average reward, that is, $\rho^*=\max_{\pi\in\Pi}\rho^\pi$. There exist constants $L_2^\Pi$ and $C_{PL}$ which are determined by the underlying MDP such that the following is true,
    \begin{equation}
        \rho^*-\rho^{\pi_{k+1}}\leq\max\pbr{\frac{128L_2^\Pi C_{PL}^2}{k},2^{-\frac{k}{2}}\pbr{\rho^*-\rho^{\pi_0}}}.
    \end{equation}
\end{theorem}
The proof of the theorem and all the supporting lemmas can be found in the Appendix. 

\subsection{Key Ideas and Proof outline}
A similar result was proved for discounted reward MDPs in \citep{agarwal2020theory,PGConvRate}. An important property pivotal to the global convergence analysis of the projected policy gradient is the smoothness of the discounted reward value function. Demonstrating the smoothness of the discounted reward value function is relatively straightforward due to the contractive properties of the discount factor. However, this poses a significant challenge in the context of average reward MDPs. Here, the absence of a discount factor as a source of contraction, coupled with the lack of uniqueness in the average reward value function, complicates the task of proving the smoothness of the average reward.  Therefore, the first important property we prove is the smoothness of average reward. 

\subsubsection{Smoothness of average reward}

A differentiable function $f:\mathcal{C}\to\R$ is called $L$-smooth if it satisfies
\begin{equation}
    \lVert\nabla f(y)-\nabla f(x)\rVert_2 \leq L\lVert y-x\rVert_2 \qquad \forall y,x\in \mathcal{C}.
\end{equation}
where $\mathcal{C}$ is some subset of $\R^n$. Further is the function is $L$-smooth, it satifies the following property.
\begin{equation}
    \Bigg\lvert f(y)-f(x)- \innorm{\nabla f(x),y-x}\Bigg\rvert\leq \frac{L}{2}\norm{y-x}^2\qquad \forall y,x\in \mathcal{C},
\end{equation}

From the above definition, it is apparent that if $f$ is $L$-smooth then $cf$ is $|c|L$-smooth for any $c\in\R$ \citep{boyd2004convex}. It can be shown that the infinite horizon discounted reward $V^\pi_{\mu,\gamma}$ is $\frac{2\gamma|\A|}{(1-\gamma)^3}$-smooth. Leveraging the result in \eqref{avg-disc}, one can see that the smoothness constant of $\rho^\pi$ is $\lim_{\gamma\to 1}\frac{2\gamma|\A|}{(1-\gamma)^2}\to\infty.$ Hence, the smoothness of the discounted reward cannot be leveraged to show the smoothness of the average reward. 

In this paper, we establish the smoothness of the infinite horizon average reward by first establishing the smoothness of the associated relative value function. We then leverage the average reward Bellman Equation \eqref{bellmaneq} to establish the smoothness of the average reward in terms of the smoothness of the relative value function. However, since the relative value function is unique up to an additive constant, we consider the projection of the value function onto the subspace orthogonal to the $\mathbbm{1}$ vector. This provides us with an unique representation of the value function whose smoothness can be evaluated. 

 Let $\Phi\in\R^{|\S|\times|\S|}$ be the projection matrix that maps any vector to its orthogonal projection in the euclidean norm onto the subspace perpendicular to the $\mathbbm{1}$ vector. Then the following lemma holds. 

\begin{lemma}
    Let $I$ be the identity matrix and $\mathbbm{1}$ be the all ones vector, both of dimension $|\S|$. Then the orthogonal projection matrix is given by $\Phi = \pbr{I-\frac{\mathbbm{1}\mathbbm{1}^\top}{S}}$. The unique value function $v^\pi_\phi$ is obtained as a solution to the following fixed point equation,
    \begin{equation}
        v^\pi_\phi = \Phi \pbr{r^\pi + \P^\pi v^\pi_\phi  -\rho^\pi\mathbbm{1}}
    \end{equation}
    and can be alternatively represented as
    \begin{equation}\label{proj:valfunc}
        v^\pi_\phi = \pbr{I-\Phi\P^\pi}^{-1}\Phi r^\pi.
    \end{equation}
\end{lemma}

Since $\P^\pi$ has 1 as its Perron Frobenius eigenvalue, $\pbr{I-\P^\pi}$ is a singular matrix. It can be verified that $\Phi\mathbbm{1} = 0$, hence $\mathbbm{1}$ is an eigenvector of $\Phi\P^\pi$ for all $\pi$ with a corresponding eigenvalue of $0$. It can subsequently be proven that the rest of the eigenvalues of $\Phi\P^\pi$ are all less than one in terms of their absolute value and hence \eqref{proj:valfunc} is well defined. 

With a unique closed form for the average reward value function established, the subsequent task is to determine its smoothness constant. Given that the smoothness constant of a function $f$ corresponds to the largest eigenvalue of its Hessian, we adopt an analytical approach similar to that presented in \citep{agarwal2020theory}. This involves utilizing directional derivatives and evaluating the maximum rate of change of derivatives across all directions within the policy space. It's important to note that since we are maximizing over directions expressible as differences between any two policies within the policy space, the resulting Lipschitz and smoothness constants are referred to as the restricted Lipschitz and smoothness constants, respectively. The restricted smoothness of the average reward value function is stated below.

\begin{lemma}
    For any policy $\pi\in\Pi$, there exist constants $C_m,C_p,C_r,\kappa_r\in\R^+$ which are determined by the underlying MDP, such that the value function $v^\pi_\phi$ is $4\pbr{2C^3_mC^2_p\kappa_r + C^2_mC_pC_r}$-smooth.
\end{lemma}

Since the average reward value function is Lipschitz and smooth with respect to its policy, one can directly utilize this property to establish the Lipschitzness and smoothness of the average reward. These results are characterized in the following lemmas. 

\begin{lemma}
     For any policy $\pi\in\Pi$, there exist constants $C_m,C_p,C_r,\kappa_r\in\R^+$ which are determined by the underlying MDP, such that the average reward $\rho^\pi$ is $L^\Pi_1$-Lipschitz.
     \begin{equation}
         \Bigg\lvert\Big\langle\frac{\partial\rho^{\pi}}{\partial\pi},\pi'-\pi\Big\rangle\Bigg\rvert \leq L^\Pi_1\norm{\pi'-\pi}_2,\qquad\forall \pi,\pi'\in\Pi,
     \end{equation}
where $L^\Pi_1 = 2(C_r +C_pC_m\kappa_r+ 2(C^2_mC_p\kappa_r+C_mC_r))$
\end{lemma}
The restricted Lipschitzness of the average reward is utilized to prove its restricted smoothness. 
\begin{lemma}\label{lemma4}
    For any policy $\pi\in\Pi$, there exist constants $C_m,C_p,C_r,\kappa_r\in\R^+$ which are determined by the underlying MDP, such that the average reward $\rho^\pi$ is $L^\Pi_2$-smooth.
    \begin{equation}
        \Bigg\lvert\Big\langle\pi'-\pi,\frac{\partial^2\rho^{\pi}}{\partial\pi^2}(\pi'-\pi)\Big\rangle\Bigg\rvert\leq\frac{L^\Pi_2}{2}\norm{\pi'-\pi}^2_2 \qquad \forall \pi,\pi'\in\Pi,
    \end{equation}
    where $L^\Pi_2=4(C^2_pC^2_m\kappa_r+C_pC_mC_r + (C_p+1)(C^2_mC_p\kappa_r+C_mC_r)+4(C^3_mC^2_p\kappa_r+C^2_mC_pC_r))$.
\end{lemma}

Note that the restricted Lipschitz constant of the average reward is upper bounded by its general Lipschitz constant:
\begin{equation}
    \max_{\pi'\in\Pi:\norm{\pi'-\pi}_2\leq1}\Bigg\lvert\Big\langle\frac{\partial\rho^{\pi}}{\partial\pi},\pi'-\pi\Big\rangle\Bigg\rvert \leq \max_{u\in\R^{\St\times\A}:\norm{u}_2\leq1}\Bigg\lvert\Big\langle\frac{\partial\rho^{\pi}}{\partial\pi},u\Big\rangle\Bigg\rvert 
\end{equation}
By confining our analysis of the smoothness constants to the policy class, we introduce a dependency of our convergence bounds on MDP-specific constants, including $C_r,C_p,C_m,C_e$ and $\kappa_r$. These constants capture the complexity of the underlying MDP and are exclusive to the analysis presented in this paper, as there appears to be no such dependency observed in the global convergence bounds of \citep{agarwal2020theory}. A more detailed description of these constants can be found in Table~\ref{tab:my_label}, where $k_1$ and $k_2$ represent MDP-independent numeric constants. These constants, which rely on the characteristics of the MDP, suggest that the projected policy gradient may achieve faster convergence in MDPs with lower complexity as opposed to those with higher complexity. The range of these constants can be found in Appendix \ref{appendix_A}. We now proceed to analyze the convergence of projected policy gradient utilizing the smoothness of the average reward.  

\begin{table}[h]
\caption{Constants capturing the MDP Complexity}
\label{tab:my_label}
    \centering
\begin{tabular}{c|p{6cm}cc}
\toprule
   & Definition& Range &Remark\\
\midrule
$C_m$&$\max_{\pi\in\Pi}\norm{(I-\Phi \P^\pi)^{-1}}$&$\frac{2C_e|\S|}{1-\lambda}$& Lowest rate of mixing\\\\
$C_p$&$\max_{\pi,\pi'\in\Pi}\frac{\norm{\P^{\pi'}-\P^\pi}}{\norm{\pi'-\pi}_2}$& $[0,\sqrt{|\A|}]$&Diameter of transition kernel\\\\
$C_r$&$\max_{\pi,\pi'}\frac{\norm{r^{\pi'}-r^\pi}_\infty}{\norm{\pi'-\pi}_2}$& $[0,\sqrt{|\A|}]$& Diameter of reward function\\\\
$\kappa_r$&$ \max_{\pi}\norm{\Phi r^\pi }_\infty$& $[0,2)$& Variance of reward function\\\\
$\frac{L^\Pi_1}{2}$&$C_r +C_pC_m\kappa_r+ 2(C^2_mC_p\kappa_r+C_mC_r)$&$[0,k_1\sqrt{|\A|}C^2_m]$& Restricted Lipschitz constant\\\\
$\frac{L^\Pi_2}{4}$&$C^2_pC^2_m\kappa_r+C_pC_mC_r
+ (C_p+1)(C^2_mC_p\kappa_r+C_mC_r)+4(C^3_mC^2_p\kappa_r+C^2_mC_pC_r)$&$[0,k_2|\A|C^3_m]$&Restricted smoothness constant\\
\bottomrule
\end{tabular}
\begin{tabular}{c}
\end{tabular}
\end{table}

\subsubsection{Convergence of policy gradient}
Using the smoothness property of the average reward, it is possible to show that the improvement in the successive average reward iterates is bounded from below by the product of the smoothness constant and the difference in the policy iterates, as described in the Lemma below.
\begin{lemma}\label{lemma5}
Let $\rho^{\pi_k}$ be the average reward corresponding to the policy iterate $\pi_k$ obtained from \eqref{avg:pg}. Let $L^\Pi_2$ be as in Lemma \ref{lemma4}. Then,
\begin{equation}
    \rho^{\pi_{k+1}}-\rho^{\pi_k} \geq \frac{L^\Pi_2}{2}\norm{\pi_{k+1}-\pi_k}^2,\qquad \forall k\in\N.
\end{equation}
\end{lemma}
Successively increasing iterates are not sufficient to guarantee finite time global convergence bounds. It is therefore necessary to bound the suboptimality associated with each iterate. We do so by levaraging the performance difference lemma stated below. 
\begin{lemma} \label{lemma6}
Let $\rho^*$ be the globally optimal average reward. Then for any $\pi\in\Pi$, the suboptimality of $\rho^\pi$ can be expressed as:
\begin{equation}
\label{eq:pdl}
    \rho^*-\rho^\pi = \sum_{s}d^{\pi^*}(s){Q}^{\pi}(s,a)[\pi^*(a|s)-\pi(a|s)].   
\end{equation}
\end{lemma}
\proof{The proof can be found in \citep{PDL_averageCase}}.

In the next lemma, we upper bound the right-hand side of \eqref{eq:pdl} in terms of the gradient of $\rho^\pi.$
\begin{lemma}\label{lemma7}
The suboptimality of any $\pi\in\Pi$ satisfies:
    \begin{equation}
        \rho^{*}-\rho^\pi \leq C_{PL}\max_{\pi'\in\Pi}\left\langle\pi'-\pi,\frac{\partial\rho^{\pi}}{\partial\pi}\right\rangle,\qquad \forall \pi\in\Pi,
    \end{equation}
where $C_{PL} = \max_{\substack{\pi\in\Pi\\s\in\S}}\frac{d^{\pi^*}(s)}{d^\pi(s)}$.
\end{lemma}
Note that $C_{PL}$ is a constant that is proportional to the size of the state space. We do not know if the appearance of such a constant is inevitable or not; however, it should be noted such a constant appears in prior works on discounted reward problems as well \citep{agarwal2020theory,PGConvRate}. 

It is possible to further upper bound the expression in Lemma \ref{lemma7} using the smoothness property of the average reward. 

\begin{lemma}\label{lemma8}
Let $\pi_k$ be the policy iterates generated by \eqref{avg:pg}. Then for all $\pi'\in\Pi$ it is true that,
    \begin{equation}
         \Big\langle{ \frac{\partial \rho_{\pi_{k+1}}}{\partial \pi_{k+1}}, \pi'-\pi_{k+1}} \Big\rangle \leq  4\sqrt{|\S|}L^\Pi_2\norm{\pi_{k+1}-\pi_k},
    \end{equation}
\end{lemma}
Lemmas \ref{lemma5},\ref{lemma7} and \ref{lemma8} are combined to prove the result in Theorem \ref{maintheorem}.

\subsection{Extension to Discounted Reward MDPs}
Existing performance bounds in the context of discounted reward MDPs require an \textcolor{black}{iteration} complexity of $O\pbr{\frac{|\S||\A|}{(1-\gamma)^5\epsilon}}$  to achieve policies with suboptimality of $\epsilon$ \citep{PGConvRate}. Adapting a similar analysis as presented in this paper results in a iteration complexity of $O\pbr{\frac{|\S|L}{(1-\gamma)^5\epsilon}}$, where $0\leq L\leq|\A|.$ The constant $L$ captures the hardness of the MDP. Therefore, MDPs with lower complexity, i.e., lower values of $L$, converge faster than MDPs with higher complexity, thus improving on current complexity-independent bounds. More details can be found in the Appendix \ref{appendix_C}.


\section{Simulations}
Here, we present simulations corresponding to two MDP complexity measures. 
\vspace{-5mm}
\begin{figure}[H]
    \centering
    \subfigure[Convergence as a function of state space cardinality]{
    \includegraphics[width=0.48\textwidth]{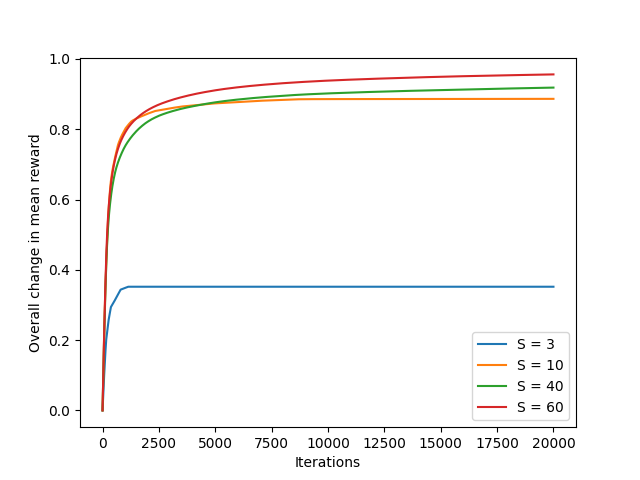}
    \label{fig1m}
    }
    \hfill
    \subfigure[Convergence as a function of reward diameter]{
    \includegraphics[width=0.48\textwidth]{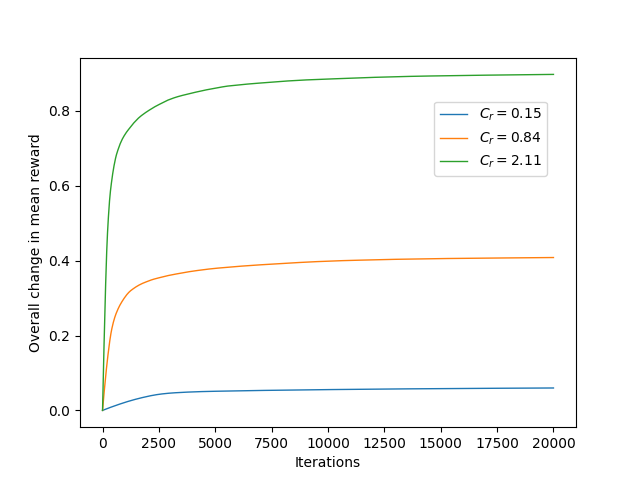}
    \label{fig2m}
    }
    \caption{Improvement in average reward as a function of MDP complexity}
\end{figure}

In Subfigure \ref{fig1m}, we simulated MDPs with $(|\S|,|\A|)=\{(3,3), (10,10), (40,40), (60,60)\}$, whose transition kernels are randomly generated. The reward matrix is generated from a uniform distribution over $[-1,1]$. Projected policy gradient was implemented for $20000$ iterations and the overall change in the average reward is plotted as a function of iteration number. As expected, the convergence rate is slower when $(|\S|,|\A|)$ are larger due to the fact that the reward smoothness constant is larger. This reduction stems from lower values of $C_M, C_r, C_p$, which are characteristic of MDPs with smaller state and action space cardinalities. A less obvious result is that, even for MDPs with a fixed size of state and action spaces, the rate of convergence can be considerably different as shown in
Subfigure \ref{fig2m}. For this simulation, we fix the state and action space cardinality at $(|\S|,|\A|)=(20,20).$ We randomly generate a transition kernel, which remains constant across different single-step reward functions corresponding to varying reward diameters. The observed convergence trend aligns with the theoretical bounds obtained, indicating that MDPs with lower values of $C_r$  tend to converge relatively faster. Similar results are presented in the Appendix \ref{appendix_D} for other measures of MDP complexity. This observation appears to be new; the performance bounds in prior works on discounted reward problems do not seem to capture the role of MDP complexity and are only a function of the size of state and action spaces.

\section*{Acknowledgements}
Research conducted by Y.M. and R.S. was supported in part by NSF Grants CNS 23-12714, CCF 22-07547, CNS 21-06801, and AFOSR Grant FA9550-24-1-0002. 

\bibliographystyle{plainnat}
\bibliography{main}

\newpage
\appendix
\addcontentsline{toc}{section}{Appendix} 
\part{Appendix} 
\parttoc 

\section{Smoothness of Average Reward}
\label{appendix_A}
\subsection{Proof of Lemma 1}
Consider the subspace orthogonal $E$ to the all ones vector $\mathbbm{1}\in\R^{|\S|}$ defined below:
\begin{equation}
    E = \text{span}\cbr{\theta\in\R^{|\S|}:\theta^\top\mathbbm{1}=0.}
\end{equation}
The orthogonal projection $v_\phi$ of a vector $v$ in the Euclidean norm onto the subspace $E$ is defined as:
\begin{equation}
    v_\phi = \argmin_{u\in\E}||v-u||_2
\end{equation}
It can be checked that the closed form expression for $v_\phi$ is given by:
\begin{equation}
    v_\phi = \pbr{I-\frac{\mathbbm{1}\mathbbm{1}^\top}{|\S|}}v
\end{equation}
where $I\in\R^{|\S|\times|\S|}$ is the identity matrix. 

Consider the projection of the vector $r^\pi+\P^\pi v^\pi - \rho^\pi\mathbbm{1}$ onto $E$ for any policy $\pi\in\Pi$. The above projection is identical to the projection of $r^\pi+\P^\pi v^\pi$ onto $E$, since $\rho^\pi\mathbbm{1}$ lies in the nullspace of $\Phi$. 
\begin{align}
    \Phi(r^\pi -\rho^\pi \mathbbm{1}+\P^\pi v)
    =& \pbr{r^\pi -\rho^\pi \mathbbm{1}+\P^\pi v} -\big\langle\mathbbm{1}, r^\pi -\rho^\pi \mathbbm{1}+\P^\pi v\big\rangle\frac{\mathbbm{1}}{|\S|} \\
    =& \pbr{r^\pi +\P^\pi v} -\big\langle\mathbbm{1}, r^\pi +\P^\pi v\big\rangle\frac{\mathbbm{1}}{|\S|}\\
    =& r^\pi-\langle r^\pi,\mathbbm{1}\rangle\frac{\mathbbm{1}}{|\S|} +\P^\pi v -\langle\mathbbm{1},\P^\pi v\rangle\frac{\mathbbm{1}}{|\S|}\\
    =& r^\pi-\langle r^\pi,\mathbbm{1}\rangle\frac{\mathbbm{1}}{|\S|} +\P^\pi v -\frac{\mathbbm{1}}{|\S|}(\mathbbm{1}^\top\P^\pi v)\\
    =& (I  -\frac{\mathbbm{1}\mathbbm{1}^\top}{|\S|})r^\pi +(I  -\frac{\mathbbm{1}\mathbbm{1}^\top}{|\S|})\P^\pi v\\
    =& \Phi\sbr{r^\pi +\P^\pi v}.
\end{align}
Consider the average reward Bellman equation corresponding to policy $\pi\in\Pi$:
\begin{equation}
\label{appen_acoe}
    \rho^\pi\mathbbm{1} + v^\pi = r^\pi + \P^\pi v^\pi
\end{equation}
Imposing an additional constraint ${v^\pi}^\top \mathbbm{1}=0$ yields a unique average reward value function denoted by $v^\pi_\phi$. Moreover, it is true that,
 \begin{align}
 \Phi v_\phi^\pi + \Phi \rho^\pi\mathbbm{1} =& \Phi r^\pi + \Phi  \P^\pi v_\phi^\pi \\
 \implies \Phi v_\phi^\pi =& \Phi r^\pi + \Phi  \P^\pi v_\phi^\pi,\\
\stackrel{(\text{a})}{\implies} v_\phi^\pi =& \Phi[ r^\pi +   \P^\pi v_\phi^\pi],
\end{align} 
where (a) is true because ${v^\pi}^\top \mathbbm{1}=0 \implies\Phi v^\pi = v^\pi$.
Thus the projected value function with an unique representation is given by:
\begin{equation}
\label{val_func_proj}
    v_\phi^\pi = \sbr{I-\Phi\P^\pi}^{-1}\Phi r^\pi, 
\end{equation}
and the existence of the inverse is proven in Subsection \ref{subs:ev}, Lemma~\ref{evlemma}. An alternate expression for the projected value function is given by: $v_\phi^\pi = \pbr{I+\mathbbm{1}\mathbbm{1}^\top D-\frac{\mathbbm{1}\mathbbm{1}^\top}{|\S|}}v_0^\pi$, where $D\in\R^{|\S|\times|\S|}$ is a diagonal matrix whose entries correspond to the stationary measure over the states associated with policy $\pi$. See \citep{tsitsiklis1999average} for more details. 

\subsection{Proof that Eigenvalues of \texorpdfstring{$\pbr{I-\Phi\P^\pi}$}{Lg} are Non-zero}\label{subs:ev}

In this subsection, we introduce the lemmas required to establish the proof of the eigenvalues of $\left(I-\frac{\mathbbm{1}\mathbbm{1}^\top}{|\S|}\right)\P^\pi$ being nonzero. We use the following notation: $\mathbbm{1}\in\R^n$ represents the all ones vector and $I\in\R^{n\times n}$ is the identity matrix.
\begin{lemma}\label{matrix_power}
    Let $A\in\R^{n\times n}$ be a stochastic matrix. It is true that
    \begin{equation}
        \pbr{\pbr{I  -\frac{\mathbbm{1}\mathbbm{1}^\top}{n}}A}^k = \pbr{I  -\frac{\mathbbm{1}\mathbbm{1}^\top}{n}}A^k .
    \end{equation}
\end{lemma}
\begin{proof}
    For any $k\in\N$, consider, 
    \begin{align}
      \pbr{I  -\frac{\mathbbm{1}\mathbbm{1}^\top}{n}}A^k \pbr{I  -\frac{\mathbbm{1}\mathbbm{1}^\top}{n}}A&=\pbr{A^k  -\frac{\mathbbm{1}\mathbbm{1}^\top}{n}A^k} \pbr{A  -\frac{\mathbbm{1}\mathbbm{1}^\top}{n}A} \\   &=A^{k+1}  -\frac{\mathbbm{1}\mathbbm{1}^\top}{n}A^{k+1}   -A^k\frac{\mathbbm{1}\mathbbm{1}^\top}{n}A + \frac{\mathbbm{1}\mathbbm{1}^\top}{n}A^{k}\frac{\mathbbm{1}\mathbbm{1}^\top}{n}A \\
      &\stackrel{\text{(a)}}{=}A^{k+1}  -\frac{\mathbbm{1}\mathbbm{1}^\top}{n}A^{k+1}   -\frac{\mathbbm{1}\mathbbm{1}^\top}{n}A + \frac{\mathbbm{1}\mathbbm{1}^\top}{n}\frac{\mathbbm{1}\mathbbm{1}^\top}{n}A, \\&\stackrel{\text{(b)}}{=}A^{k+1}  -\frac{\mathbbm{1}\mathbbm{1}^\top}{n}A^{k+1}   -\frac{\mathbbm{1}\mathbbm{1}^\top}{n}A + \frac{\mathbbm{1}\mathbbm{1}^\top}{n}A, \\&=A^{k+1}  -\frac{\mathbbm{1}\mathbbm{1}^\top}{n}A^{k+1}. \\
      &=\pbr{I  -\frac{\mathbbm{1}\mathbbm{1}^\top}{n}}A^{k+1}
    \end{align}
where (a) is true because $A^k\mathbbm{1}=\mathbbm{1}$ and (b) follows from the fact that $\frac{\mathbbm{1}^\top\mathbbm{1}}{n} =\mathbbm{1}$.
From mathematical induction it thus follows that, 
\begin{equation}
    \pbr{\pbr{I  -\frac{\mathbbm{1}\mathbbm{1}^\top}{n}}A}^k =\pbr{I  -\frac{\mathbbm{1}\mathbbm{1}^\top}{n}}A^k \qquad\forall k\in\N.
\end{equation}
\end{proof}

\begin{lemma}\label{zeromatrix}
    For any irreducible and aperiodic stochastic matrix $A\in\R^{n\times n}$, it is true that
    \begin{equation}
        \lim_{k\to\infty}\pbr{\pbr{I  -\frac{\mathbbm{1}\mathbbm{1}^\top}{n}}A}^k =0
    \end{equation}
\end{lemma}
\begin{proof}
    From Lemma \ref{matrix_power} we have,
    \begin{equation}
        \pbr{\pbr{I  -\frac{\mathbbm{1}\mathbbm{1}^\top}{n}}A}^k = \pbr{I  -\frac{\mathbbm{1}\mathbbm{1}^\top}{n}}A^k \qquad \forall k\in\N
    \end{equation}
Since $A$ is irreducible and aperiodic, the following limit converges to the stationary distribution $d\in\R^n_+$ associated with $A$. 
\begin{equation}\label{stat}
    \lim_{k\to\infty}A^k = \mathbbm{1}d^\top
\end{equation}
Consider the following,
\begin{align}
\lim_{k\to\infty}\pbr{\pbr{I  -\frac{\mathbbm{1}\mathbbm{1}^\top}{n}}A}^k =&\lim_{k\to\infty}\pbr{I  -\frac{\mathbbm{1}\mathbbm{1}^\top}{n}}A^k \\
=&\pbr{I  -\frac{\mathbbm{1}\mathbbm{1}^\top}{n}}\lim_{k\to\infty}A^k \\
=&\pbr{I  -\frac{\mathbbm{1}\mathbbm{1}^\top}{n}}\mathbbm{1}d^\top \qquad \text{(from Equation~\eqref{stat})},\\ 
=& \mathbbm{1}d^\top  -\frac{\mathbbm{1}\mathbbm{1}^\top}{n}\mathbbm{1}d^\top \\
\stackrel{\text{(a)}}{=}& \mathbbm{1}d^\top  -\mathbbm{1}d^\top \\
=& 0.
\end{align}
where (a) is true because $\frac{\mathbbm{1}^\top\mathbbm{1}}{n}=1$.
\end{proof}

\begin{lemma}
\label{inverse}
    Let $A\in\R^{n\times n}$ be a matrix such that $\lim_{k\to\infty}A^k =0$. Then $(I-A)^{-1} = \sum_{k=0}^{\infty}A^k$.
\end{lemma}
\begin{proof}
For any $K\in\N$, consider the following,
\begin{align}
\pbr{I-A}\pbr{\sum_{k=0}^{K}A^{k}} &= I-A^{K+1},\\
\implies(I-A)\pbr{\lim_{K\to\infty}\sum_{k=0}^{K}A^{k}} =& \lim_{K\to\infty}\pbr{I-A^{K+1}} \stackrel{\text{(a)}}{=} I,
\end{align}
where (a) follows from the fact that $\lim_{k\to\infty}A^k =0$. Hence the inverse of $\pbr{I-A}$ can be expressed as
$(I-A)^{-1} = \sum_{k=0}^{\infty}A^k$.
\end{proof}

\begin{lemma}
\label{evlemma}
    Let $A\in\R^{n\times n}$ be an irreducible and aperiodic stochastic matrix. Then the matrix $\pbr{I-\pbr{I  -\frac{\mathbbm{1}\mathbbm{1}^\top}{n}}A}$ is invertible and its inverse is given by:
    \begin{equation}
        \pbr{I-\pbr{I  -\frac{\mathbbm{1}\mathbbm{1}^\top}{n}}A}^{-1} = \sum_{k=0}^{\infty}\pbr{I  -\frac{\mathbbm{1}\mathbbm{1}^\top}{n}}A^k 
    \end{equation}
\end{lemma}
\begin{proof}
    Let $\lambda_i$ be eigenvalues of $\pbr{I  -\frac{\mathbbm{1}\mathbbm{1}^\top}{n}}A$. Then $\lambda_i^k$ represents the eigenvalues of $\pbr{\pbr{I  -\frac{\mathbbm{1}\mathbbm{1}^\top}{n}}A}^k$.
    But from Lemma \ref{zeromatrix}, we know that 
    \begin{equation}
        \lim_{k\to\infty}\pbr{\pbr{I  -\frac{\mathbbm{1}\mathbbm{1}^\top}{n}}A}^k = 0
    \end{equation}
Since eigenvalues are continuous functions of their corresponding matrices and all eigenvalues of a zero matrix are zero, we thus have,
\begin{equation}\label{eqval}
    \lim_{k\to\infty}\lambda^k_i = 0 \qquad \forall i\in\{1,\ldots,n\}
\end{equation}
Equation~\ref{eqval} thus implies that $|\lambda_i|< 1, \forall i\in\{1,\ldots,n\}.$ Hence the matrix $\pbr{I-\pbr{\pbr{I -\frac{\mathbbm{1}\mathbbm{1}^\top}{n}}A}}$ has all non zero eigenvalues and is thus invertible. From Lemma \ref{inverse}, we know that 
\begin{equation}
    (I-A)^{-1} = \sum_{k=0}^{\infty}A^k
\end{equation}
when $\lim_{k\to\infty}A^k = 0.$ Since, $\lim_{k\to\infty}\pbr{\pbr{I  -\frac{\mathbbm{1}\mathbbm{1}^\top}{n}}A}^k =0$ from Lemma \ref{zeromatrix}, we have the following result,
\begin{equation}
    \pbr{I-\pbr{I  -\frac{\mathbbm{1}\mathbbm{1}^\top}{n}}A}^{-1} = \sum_{k=0}^{\infty}\pbr{I  -\frac{\mathbbm{1}\mathbbm{1}^\top}{n}}A^k 
\end{equation}
\end{proof}
From definition we have $\Phi = \pbr{I  -\frac{\mathbbm{1}\mathbbm{1}^\top}{n}}$. Hence the inverse $\pbr{1-\Phi\P^\pi}^{-1} = \sum_{k=0}^\infty\Phi\pbr{\P^\pi}^k$ exists and is well defined for all $\pi\in\Pi.$ 

\subsection[Proof of Lemma 2]{Smoothness of the Average Reward Value Function $v^\pi_\phi$}

In order to prove the smoothness of the average reward value function and the infinite horizon average reward, we consider an analysis inspired by \citep{agarwal2020theory}, where instead of computing the maximum eigenvalue of the associated Hessian matrices, we consider the maximum value of the directional derivative across all directions within the policy class. 

Let $\pi, \pi'\in\Pi$ be any policies within the policy class. Then define $\pi_\alpha$ as a convex combination of policies $\pi$ and $\pi'$. That is
\begin{align}
    \pi_\alpha :&= (1-\alpha)\pi + \alpha \pi' \\ &= \pi + \alpha(\pi'-\pi) \\& =\pi +\alpha u
\end{align}
where $u = \pi'-\pi$.

Since $\pi_\alpha$ is linear in $\alpha$, it is true that
\begin{align}
\nabla_\alpha\pi_\alpha = \frac{d(\pi +\alpha u)}{d\alpha} = u,\qquad\text{and}\qquad 
\nabla^2_\alpha\pi_\alpha
=0.
\end{align}
This thus implies,
\begin{align}
\norm{\nabla_\alpha\pi_\alpha}_2 = \norm{u}_2\leq \norm{\pi'-\pi}_1\leq 2S,\qquad \text{and}\qquad \norm{\nabla^2_\alpha\pi_\alpha}_2 = 0,
\end{align}
Thus, $\pi_\alpha$ is both $\norm{u}_2$-Lipschitz and $0$-smooth with respect to $\alpha$, for all $u$ that can be represented as the difference of any two policies.

From the definition of $\P^\pi$, we have
\begin{align}
     \P^{\pi_\alpha}(s'|s) &= \sum_{a\in\A}\pi_\alpha(a|s)\P(s'|s,a)
    \\
    &= \sum_{a\in\A}\sbr{\pi(a|s)+\alpha u(a|s)} \P(s'|s,a)
    \\
   \implies \frac{\partial \P^{\pi_\alpha}(s'|s)}{\partial\alpha}
    &= \sum_{a\in\A} u(a|s)\P(s'|s,a).
\end{align}
That is,
\begin{align}\label{eq:derPpi}
\nabla_\alpha \P^{\pi_\alpha} = \P^u,\qquad\text{consequently}\qquad\nabla^2_\alpha \P^{\pi_\alpha}=0.
\end{align}
From the definition of $r^\pi$, we have
\begin{align}
r^{\pi_\alpha}(s) &= \sum_{a\in\A}\pi_\alpha(a|s)r(s,a)\\
&= \sum_{a\in\A}\sbr{\pi(a|s)+\alpha u(a|s)}r(s,a)\\
\implies \frac{\partial r^{\pi_\alpha}(s)}{\partial\alpha}
&= \sum_{a\in\A} u(a|s)r(s,a).
\end{align}
That is,
\begin{align}\label{eq:derRpi}
\nabla_\alpha r^{\pi_\alpha} = r^u,\qquad\text{consequently}\qquad\nabla^2_\alpha r^{\pi_\alpha}=0.
\end{align}
Hence the policy $\pi_\alpha$, the associated reward $r^{\pi_\alpha}$ and the transition kernel $\P^{\pi_\alpha}$ are all Lipschitz and smooth with respect to $\alpha.$

\begin{lemma}
\label{M_def}
    Let $A(\alpha)\in\R^{n\times n}$ be a matrix such that $\pbr{I-A(\alpha)}$ is invertible for all $\alpha\in\sbr{0,1}$. Define $M(\alpha):=\pbr{I-A(\alpha)}^{-1}$. Then it is true that,
    \begin{equation}
        \frac{\partial^2 M(\alpha)}{\partial \alpha^2} =\frac{\partial M(\alpha)}{\partial\alpha}\frac{\partial A(\alpha)}{\partial \alpha}M(\alpha) + M(\alpha)\frac{\partial^2 A(\alpha)}{\partial \alpha^2}M(\alpha) + M(\alpha)\frac{\partial A(\alpha)}{\partial \alpha}\frac{\partial M(\alpha)}{\partial \alpha}.
    \end{equation}
\end{lemma}
\begin{proof}
\begin{align}
        M(\alpha)\pbr{I-A(\alpha)}&=I \\
    \frac{\partial M(\alpha)}{\partial\alpha}\pbr{I-A(\alpha)} - M(\alpha)\frac{\partial A(\alpha)}{\partial\alpha}&=0 \\
    \frac{\partial M(\alpha)}{\partial\alpha} &= M(\alpha)\frac{\partial A(\alpha)}{\partial\alpha}M(\alpha) \\
    \frac{\partial^2 M(\alpha)}{\partial \alpha^2} &= \frac{\partial}{\partial\alpha}\pbr{M(\alpha)\frac{\partial A(\alpha)}{\partial \alpha}M(\alpha)},\\
 &= \frac{\partial M(\alpha)}{\partial\alpha}\frac{\partial A(\alpha)}{\partial \alpha}M(\alpha) + M(\alpha)\frac{\partial^2 A(\alpha)}{\partial \alpha^2}M(\alpha) + M(\alpha)\frac{\partial A(\alpha)}{\partial \alpha}\frac{\partial M(\alpha)}{\partial \alpha}
\end{align}
\end{proof}

Consider the following definition utilized in the proofs of the upcoming lemmas. 
\begin{equation}
\label{def_of_M}
    M^{\pi_\alpha}=\sbr{I-\Phi\P^{\pi_\alpha}}^{-1}
\end{equation}

\begin{lemma}
\label{lipschitz_vf}
Recall the definition of the projected average reward value function $v_\phi^\pi$ in Equation~\eqref{val_func_proj}. Value function $v_\phi^\pi$ is $ 2C^2_mC_p\kappa_r+2C_mC_r$-Lipschitz in $\Pi$, that is
\begin{equation}
    \Bigg\lvert\Bigg\langle \frac{\partial v_\phi^\pi}{\partial \pi},\pi'-\pi\Bigg\rangle\Bigg\rvert \leq 2\pbr{C^2_mC_p\kappa_r+C_mC_r}\norm{\pi'-\pi}_2,\qquad\forall\pi,\pi'\in\Pi.
\end{equation}
\end{lemma}
\begin{proof}
    \begin{align}
v_\phi^{\pi_\alpha} &=  M^{\pi_\alpha}  \Phi r^{\pi_\alpha} \\
\implies \frac{\partial v_\phi^{\pi_\alpha}}{\partial\alpha
} &= \frac{\partial M^{\pi_\alpha}  }{\partial\alpha
} \Phi r^{\pi_\alpha} + M^{\pi_\alpha}  \Phi\frac{\partial   r^{\pi_\alpha}}{\partial\alpha
}\\&= M^{\pi_\alpha}  \frac{\partial \Phi \P^{\pi_\alpha}}{\partial\alpha
} M^{\pi_\alpha}  \Phi r^{\pi_\alpha} + M^{\pi_\alpha}  \Phi\frac{\partial   r^{\pi_\alpha}}{\partial\alpha
},\qquad\text{(from Lemma \ref{M_def})},\\&= M^{\pi_\alpha}  \Phi \P^u M^{\pi_\alpha}  \Phi r^{\pi_\alpha} + M^{\pi_\alpha} \Phi r^u,\qquad\text{(from \eqref{eq:derPpi} and \eqref{eq:derRpi})}.\\
\implies\Big\|{\frac{\partial v_\phi^{\pi_\alpha}}{\partial\alpha}}\Big\|_{\infty} &= \norm{M^{\pi_\alpha}  \Phi \P^u M^{\pi_\alpha}\Phi  r^{\pi_\alpha} + M^{\pi_\alpha}  \Phi r^u}_{\infty}\\
&\leq \norm{M^{\pi_\alpha}  }_\infty\norm{\Phi}_\infty\norm{ \P^u}_\infty\norm{ M^{\pi_\alpha}  }_\infty\norm{\Phi r^{\pi_\alpha}}_{\infty} + \norm{M^{\pi_\alpha}}_\infty\norm{ \Phi r^{u}}_{\infty}\\
&\leq 2 C^2_mC_p\kappa_r+2C_mC_r.
\end{align}
The constants $C_m, C_p, C_r$ and $\kappa_r$ are characterized in Table 1 with their respective bounds in Lemma \ref{theconstants}.
\end{proof}

We can now build on the previous lemma to prove the smoothness of the average reward value function.

\begin{lemma}
    The value function $v_\phi^\pi$ is $ 8(C^3_mC^2_p\kappa_r + C^2_mC_pC_r)$-smooth in $\Pi$. That is, 
    \begin{equation}
         \Bigg\langle\pi'-\pi,\frac{\partial^2v_\phi^{\pi}(s)}{\partial \pi}(\pi'-\pi)\Bigg\rangle\leq 8\pbr{C^3_mC^2_p\kappa_r + C^2_mC_pC_r}\norm{\pi'-\pi}_2^2 \qquad \forall\pi',\pi\in\Pi,s\in\St
    \end{equation}
\end{lemma}
\begin{proof}
From Lemma \ref{lipschitz_vf}, it is true that
\begin{align*}
\frac{\partial v_\phi^{\pi_\alpha}}{\partial\alpha
} =& M^{\pi_\alpha}  \Phi \P^u M^{\pi_\alpha}  \Phi r^{\pi_\alpha} + M^{\pi_\alpha} \Phi r^u\\
\implies
\frac{\partial^2 v_\phi^{\pi_\alpha}}{\partial\alpha^2
} = &\frac{\partial}{\partial\alpha}\Bigm[M^{\pi_\alpha}  \Phi \P^u M^{\pi_\alpha}  \Phi r^{\pi_\alpha} + M^{\pi_\alpha} \Phi r^u\Bigm]\\\\
=&\frac{\partial M^{\pi_\alpha} }{\partial\alpha} \Phi \P^u M^{\pi_\alpha} \Phi r^{\pi_\alpha} +M^{\pi_\alpha} \Phi \P^u \frac{\partial M^{\pi_\alpha} }{\partial\alpha} \Phi r^{\pi_\alpha} +M^{\pi_\alpha} \Phi \P^u M^{\pi_\alpha}\Phi  
\frac{\partial r^{\pi_\alpha}}{\partial\alpha} \\
&
\qquad+\frac{\partial M^{\pi_\alpha}  }{\partial\alpha}\Phi r^{u}\\\\
=& M^{\pi_\alpha}  \frac{\partial \Phi \P^{\pi_\alpha}}{\partial\alpha
} M^{\pi_\alpha} \Phi \P^u M^{\pi_\alpha} \Phi r^{\pi_\alpha} +M^{\pi_\alpha} \Phi \P^u  M^{\pi_\alpha}  \frac{\partial \Phi \P^{\pi_\alpha}}{\partial\alpha
} M^{\pi_\alpha} \Phi r^{\pi_\alpha} \\
&
\qquad+M^{\pi_\alpha} \Phi \P^u M^{\pi_\alpha}\Phi  
\frac{\partial r^{\pi_\alpha}}{\partial\alpha} +M^{\pi_\alpha}  \frac{\partial \Phi \P^{\pi_\alpha}}{\partial\alpha
} M^{\pi_\alpha}\Phi r^{u},\qquad\text{(from Lemma \ref{M_def})},\\\\
=& M^{\pi_\alpha}  \Phi \P^u M^{\pi_\alpha} \Phi \P^u M^{\pi_\alpha} \Phi r^{\pi_\alpha} +M^{\pi_\alpha} \Phi \P^u  M^{\pi_\alpha}   \Phi \P^u M^{\pi_\alpha} \Phi r^{\pi_\alpha} \\
&
\qquad+M^{\pi_\alpha} \Phi \P^u M^{\pi_\alpha} \Phi r^u+M^{\pi_\alpha}   \Phi \P^u M^{\pi_\alpha}\Phi r^{u},\qquad\text{(from \eqref{eq:derPpi} and \eqref{eq:derRpi})}.
\end{align*}

Considering the $L_\infty$ norm,
\begin{align*}
\Big\|{\frac{\partial^2 v_\phi^{\pi_\alpha}}{\partial\alpha^2}}\Big\|_\infty &= 2\norm{M^{\pi_\alpha} \Phi \P^u M^{\pi_\alpha}  \Phi \P^u M^{\pi_\alpha} \Phi r^{\pi_\alpha} 
+ M^{\pi_\alpha} \Phi \P^u M^{\pi_\alpha}\Phi r^u}_\infty\\&\leq 2\norm{M^{\pi_\alpha} \Phi \P^u M^{\pi_\alpha}  \Phi \P^u M^{\pi_\alpha} \Phi r^{\pi_\alpha}}_\infty 
+ \norm{M^{\pi_\alpha} \Phi \P^u M^{\pi_\alpha}\Phi r^u}_\infty\\
&\leq  8(C^3_mC^2_p\kappa_r + C^2_mC_pC_r).
\end{align*}
Hence, we obtain, 

\begin{equation}
    \Bigg\langle\pi'-\pi,\frac{\partial^2v_\phi^{\pi}(s)}{\partial \pi}(\pi'-\pi)\Bigg\rangle\leq 8\pbr{C^3_mC^2_p\kappa_r + C^2_mC_pC_r}\norm{\pi'-\pi}_2^2 \qquad \forall\pi',\pi\in\Pi,s\in\St
\end{equation}
\end{proof}

\subsection[Proof of Lemma 3]{Lipscitzness of the Infinite Horizon Average Reward $\rho^\pi$}
The Lipschitzness and smoothness of the projected value function $v^\pi_\phi$ is leveraged through the average reward Bellman equation to prove the Lipschitzness and smoothness of the infinite horizon average reward.

\begin{lemma}
\label{lipschitz_avreward}
    Recall the average reward Bellman Equation corresponding to a policy $\pi$ and projected value function $v_\phi^\pi$ in Equation \eqref{appen_acoe}. The average reward $\rho^\pi$ is $L^\Pi_1$-Lipschitz.
     \begin{equation}
         \Bigg\lvert\Big\langle\frac{\partial\rho^{\pi}}{\partial\pi},\pi'-\pi\Big\rangle\Bigg\rvert \leq L^\Pi_1\norm{\pi'-\pi}_2,\qquad\forall \pi,\pi'\in\Pi,
     \end{equation}
where $L^\Pi_1 = 2(C_r +C_pC_m\kappa_r+ 2(C^2_mC_p\kappa_r+C_mC_r))$
\end{lemma}
\begin{proof}
    From Equation \eqref{appen_acoe}, 
    \begin{align}
        \rho^\pi
        \mathbbm{1}= r^\pi + \P^\pi v_\phi^\pi - v_\phi^\pi.
    \end{align}
Taking derivative with respect to $\alpha$,

\begin{align}
\frac{\partial\rho^{\pi_\alpha}}{\partial\alpha}
\mathbbm{1} &= \frac{\partial r^{\pi_\alpha}}{\partial\alpha} + \frac{\partial \P^{\pi_\alpha}}{\partial\alpha}  v_\phi^{\pi_\alpha}+ \P^{\pi_\alpha}\frac{\partial v_\phi^{\pi_\alpha}}{\partial\alpha}   - \frac{\partial v_\phi^{\pi_\alpha}}{\partial\alpha}\\\\
&= \Phi r^u + \Phi \P^uv^{\pi_\alpha} + (\P^{\pi_\alpha}-I)(M^{\pi_\alpha}  \Phi \P^u M^{\pi_\alpha}  \Phi r^{\pi_\alpha} + M^{\pi_\alpha} \Phi r^u),\qquad\text{(from Lemma \ref{lipschitz_vf})}\\
&= \Phi r^u + \Phi \P^uM^{\pi_\alpha}\Phi r^{\pi_\alpha} + (\P^{\pi_\alpha}-I)(M^{\pi_\alpha}  \Phi \P^u M^{\pi_\alpha}  \Phi r^{\pi_\alpha} + M^{\pi_\alpha} \Phi r^u),\qquad\text{(from \eqref{eq:derRpi} and \eqref{eq:derPpi})}.
\end{align}

Considering the $L_\infty$ norm of the above expression, 
\begin{align}
\Big|\frac{\partial\rho^{\pi_\alpha}}{\partial\alpha}\Big| 
&= \Bigm\lVert\Phi r^u + \Phi \P^uM^{\pi_\alpha}\Phi r^{\pi_\alpha} + (\P^{\pi_\alpha}-I)(M^{\pi_\alpha}  \Phi \P^u M^{\pi_\alpha}  \Phi r^{\pi_\alpha} + M^{\pi_\alpha} \Phi r^u)\Bigm\rVert_\infty,\\
&\leq \norm{\Phi r^u}_\infty + \norm{\Phi \P^uM^{\pi_\alpha}\Phi r^{\pi_\alpha} }_\infty+ \norm{\P^{\pi_\alpha}-I}_\infty(\norm{M^{\pi_\alpha}  \Phi \P^u M^{\pi_\alpha}  \Phi r^{\pi_\alpha}}_\infty + \norm{M^{\pi_\alpha} \Phi r^u}_\infty),\\
&\leq 2\norm{r^u}_\infty + 2\norm{ \P^uM^{\pi_\alpha}\Phi r^{\pi_\alpha} }_\infty+ \norm{\P^{\pi_\alpha}-I}(\norm{M^{\pi_\alpha}  \Phi \P^u M^{\pi_\alpha}  \Phi r^{\pi_\alpha}}_\infty + \norm{M^{\pi_\alpha} \Phi r^u}_\infty),\\&\leq 2C_r +2C_pC_m\kappa_r+ 2(2C^2_mC_p\kappa_r+2C_mC_r).
\end{align}
\end{proof}

\subsection[Proof of Lemma 4]{Smoothness of the Infinite Horizon Average Reward $\rho^\pi$}

\begin{lemma}
    The average reward $\rho^\pi$ is $L^\Pi_2$-smooth.
    \begin{equation}
        \Bigg\lvert\Big\langle\pi'-\pi,\frac{\partial^2\rho^{\pi}}{\partial\pi^2}(\pi'-\pi)\Big\rangle\Bigg\rvert\leq\frac{L^\Pi_2}{2}\norm{\pi'-\pi}^2_2 \qquad \forall \pi,\pi'\in\Pi,
    \end{equation}
    where $L^\Pi_2=4(C^2_pC^2_m\kappa_r+C_pC_mC_r + (C_p+1)(C^2_mC_p\kappa_r+C_mC_r)+4(C^3_mC^2_p\kappa_r+C^2_mC_pC_r))$.
\end{lemma}
\begin{proof}
    From Lemma \ref{lipschitz_avreward}, we have
\begin{align}
\frac{\partial\rho^{\pi_\alpha}}{\partial\alpha}
\mathbbm{1} 
&= \Phi r^u + \Phi \P^uM^{\pi_\alpha}\Phi r^{\pi_\alpha} + (\P^{\pi_\alpha}-I)(M^{\pi_\alpha}  \Phi \P^u M^{\pi_\alpha}  \Phi r^{\pi_\alpha} + M^{\pi_\alpha} \Phi r^u).
\end{align}Taking the derivative again, and  repeatedly invoking Equations \eqref{eq:derPpi},\eqref{eq:derRpi} and Lemma \ref{M_def}, it follows that, 
\begin{align}
\frac{\partial^2\rho^{\pi_\alpha}}{\partial\alpha^2}
\mathbbm{1} =& 0+ \frac{\partial}{\partial\alpha}(\Phi \P^uM^{\pi_\alpha}\Phi r^{\pi_\alpha})+\frac{\partial}{\partial\alpha}\bigm((\P^{\pi_\alpha}-I)(M^{\pi_\alpha}  \Phi \P^u M^{\pi_\alpha}  \Phi r^{\pi_\alpha} + M^{\pi_\alpha} \Phi r^u)\bigm)\\
=& \Phi \P^uM^{\pi_\alpha}\Phi \P^uM^{\pi_\alpha}\Phi r^{\pi_\alpha} 
+ \Phi \P^uM^{\pi_\alpha}\Phi r^{u}+ (\P^u)(M^{\pi_\alpha}  \Phi \P^u M^{\pi_\alpha}  \Phi r^{\pi_\alpha} + M^{\pi_\alpha} \Phi r^u)\\
&\qquad+(\P^{\pi_\alpha}-I)\bigm(M^{\pi_\alpha}\Phi \P^uM^{\pi_\alpha}  \Phi \P^u M^{\pi_\alpha}  \Phi r^{\pi_\alpha} +M^{\pi_\alpha}  \Phi \P^u M^{\pi_\alpha}\Phi \P^uM^{\pi_\alpha}  \Phi r^{\pi_\alpha} \\
&\qquad+M^{\pi_\alpha}  \Phi \P^u M^{\pi_\alpha}  \Phi r^{u} + M^{\pi_\alpha}  \Phi \P^u M^{\pi_\alpha} \Phi r^u\bigm),\\
=& \Phi \P^uM^{\pi_\alpha}\Phi \P^uM^{\pi_\alpha}\Phi r^{\pi_\alpha} 
+ \Phi \P^uM^{\pi_\alpha}\Phi r^{u}+ (\P^u)(M^{\pi_\alpha}  \Phi \P^u M^{\pi_\alpha}  \Phi r^{\pi_\alpha} + M^{\pi_\alpha} \Phi r^u)\\
&\qquad+2(\P^{\pi_\alpha}-I)\bigm(M^{\pi_\alpha}\Phi \P^uM^{\pi_\alpha}  \Phi \P^u M^{\pi_\alpha}  \Phi r^{\pi_\alpha} +M^{\pi_\alpha}  \Phi \P^u M^{\pi_\alpha}  \Phi r^{u} \bigm).
\end{align}

Considering the $L_\infty$ norm of the above expression, 
\begin{align*}
\Big|{\frac{\partial^2\rho^{\pi_\alpha}}{\partial\alpha^2}}\Big|  \leq& \norm{\Phi \P^uM^{\pi_\alpha}\Phi \P^uM^{\pi_\alpha}\Phi r^{\pi_\alpha}}_\infty
+ \norm{\Phi \P^uM^{\pi_\alpha}\Phi r^{u}}_\infty + \norm{(\P^u)(M^{\pi_\alpha}  \Phi \P^u M^{\pi_\alpha}  \Phi r^{\pi_\alpha} \\
&+ M^{\pi_\alpha} \Phi r^u)}_\infty+2\norm{(\P^{\pi_\alpha}-I)\bigm(M^{\pi_\alpha}\Phi \P^uM^{\pi_\alpha}  \Phi \P^u M^{\pi_\alpha}  \Phi r^{\pi_\alpha} +M^{\pi_\alpha}  \Phi \P^u M^{\pi_\alpha}  \Phi r^{u} \bigm)}_\infty\\
\leq& 4(C^2_pC^2_m\kappa_r+C_pC_mC_r
+ (C_p+1)(C^2_mC_p\kappa_r+C_mC_r)+4(C^3_mC^2_p\kappa_r+C^2_mC_pC_r)).
\end{align*}
\end{proof}

\textbf{Remark:} The smoothness and Lipschitz constant analysis of both the average reward value functions and the infinite horizon average reward are constrained to all directions $u$, such that every $u=\pi-\pi'$ can be expressed as a difference of any two policies $\pi,\pi'\in\Pi$. Hence the smoothness and Lipschitz constants derived are restricted to the directions that can be expressed as this difference and hence are referred to as restricted smoothness/Lipschitzness. 

\subsection{Table of constants capturing MDP complexity}
We restate the table of constants and their description here for the sake of convenience. 
\begin{table}[H]
\caption{Constants capturing the MDP Complexity}
\label{tab:append_my_label}
    \centering
\begin{tabular}{c|p{6cm}cc}
\toprule
   & Definition& Range &Remark\\
\midrule
$C_m$&$\max_{\pi\in\Pi}\norm{(I-\Phi P^\pi)^{-1}}$&$\frac{2C_e|\S|}{1-\lambda}$& Lowest rate of mixing\\\\
$C_p$&$\max_{\pi,\pi'\in\Pi}\frac{\norm{P^{\pi'}-P^\pi}}{\norm{\pi'-\pi}_2}$& $[0,\sqrt{A}]$&Diameter of transition kernel\\\\
$C_r$&$\max_{\pi,\pi'}\frac{\norm{r^{\pi'}-r^\pi}_\infty}{\norm{\pi'-\pi}_2}$& $[0,\sqrt{A}]$& Diameter of reward function\\\\
$\kappa_r$&$ \max_{\pi}\norm{\Phi r^\pi }_\infty$& $[0,2]$& Variance of reward function\\\\
$\frac{L^\Pi_1}{2}$&$C_r +C_pC_m\kappa_r+ 2(C^2_mC_p\kappa_r+C_mC_r)$&$[0,k_1\sqrt{A}C^2_m]$& Restricted Lipschitz constant\\\\
$\frac{L^\Pi_2}{4}$&$C^2_pC^2_m\kappa_r+C_pC_mC_r
+ (C_p+1)(C^2_mC_p\kappa_r+C_mC_r)+4(C^3_mC^2_p\kappa_r+C^2_mC_pC_r)$&$[0,k_2AC^3_m]$&Restricted smoothness constant\\
\bottomrule
\end{tabular}
\begin{tabular}{c}
\end{tabular}
\end{table}

\begin{lemma} 
\label{theconstants}
The constants $C_p,C_r,C_m,\kappa_r$ in Table \ref{tab:append_my_label} and other operator norms are bounded as below:
\begin{enumerate} 
\item $\norm{\Phi}:= \max_{\norm{v}_\infty\leq 1}\norm{\Phi v}_\infty \leq 2$.
\item  $\norm{P^\pi} =\max_{\norm{v}_\infty\leq 1}\norm{P^\pi v}_\infty \leq \max_{\norm{v}_\infty\leq 1}\norm{v}_\infty = 1$.
\item  $\kappa_r = \max_{\pi}\norm{\Phi r^\pi}_\infty \leq 2 $
\item $C_m \leq \frac{2C_eS}{1-
\lambda}$
\item $C_p = \max_{u = \frac{\pi'-\pi}{\norm{\pi'-\pi}_2},\pi',\pi\in\Pi}\max_{\norm{v}_\infty\leq 1}\norm{P^u v}_\infty \leq \sqrt{A}.$
\item 
$C_r = \max_{u = \frac{\pi'-\pi}{\norm{\pi'-\pi}_2},\pi',\pi\in\Pi}\norm{R^u }_\infty \leq \sqrt{A}.$
\end{enumerate}
\end{lemma}

\begin{proof}
    \begin{enumerate}
        \item Consider the projection matrix $\Phi$,
        \begin{align}
            \norm{\Phi}_\infty &=\max_{\norm{v}_\infty\leq 1}\norm{\Phi v}_\infty \leq \max_{{\norm{v}}_\infty} \\ &= \max_{s\in\S}\left|{v(s)-\frac{\sum_{s\in\S}v(s)}{S}}\right|\\ & \leq \max_{s\in\S}\left|{v(s)}\right|+\left|{\frac{\sum_{s\in\S}v(s)}{S}}\right| \\ &\leq 2\norm{v}_\infty =2
        \end{align}
        \item The operator norm of $\P^\pi$ is bounded as below:
        \begin{align}
            \norm{P^\pi} &=\max_{\norm{v}_\infty\leq 1}\norm{P^\pi v}_\infty \\ &\leq \max_{\norm{v}_\infty \leq 1}\norm{v}_\infty \\ & \leq 1. 
        \end{align}
        Equality is attained by the vector $v= \mathbbm{1}.$

        \item $\kappa_r$ is bounded as below:
        \begin{align}
            \kappa_r &= \max_{\pi\in\Pi}\norm{\Phi r^\pi}_\infty \\
            &= \max_{\pi\in\Pi}\left\|{r^\pi -\frac{\sum_{s\in\S}r^\pi(s)}{|\S|}\mathbbm{1}}\right\|_\infty \\ &\leq \max_{\pi\in\Pi}\|r^\pi\|_\infty + \left\|{\frac{\sum_{s\in\S}r^\pi(s)}{|\S|}\mathbbm{1}}\right\|_\infty\\
            & \leq 2
        \end{align}
        $\kappa_r$, in some sense, captures the variance of the single step reward function across the class of policies. Greater the variation of the $r$ across different actions, greater the value of $\kappa_r.$
        \item  $C_m$ is the maximum of the operator norm of the matrix $\pbr{I-\Phi\P^\pi}^{-1}$ across all policies $\pi\in\Pi$. It is determined as follows:
        \begin{align}
    (I-\Phi \P^\pi)^{-1} &= \sum_{k=0}^{\infty}(\Phi \P^\pi)^k  = \sum_{k=0}^{\infty}\Phi \pbr{\P^\pi}^k ,\qquad \text{(from Lemma  \ref{evlemma})}, \\
&\stackrel{\text{(a)}}{=}\sum_{k=0}^{\infty}\Phi (\pbr{\P^\pi}^k-\mathbbm{1}\pbr{d^\pi}^\top) ,\qquad \text{(as $\Phi\mathbbm{1}\pbr{d^\pi}^\top = 0 $)} 
\end{align}
Let $v\in\R^{|\S|}$ such that $||v||_\infty\leq 1.$ Then,
\begin{align}
\implies  \norm{(I-\Phi \P^\pi)^{-1}v}_\infty   &\leq  \sum_{k=0}^{\infty} \norm{\Phi(\pbr{\P^\pi}^k-\mathbbm{1}\pbr{d^\pi}^\top)v}_\infty\\
&\leq  \sum_{k=0}^{\infty} \norm{\Phi}_\infty\norm{(\pbr{\P^\pi}^k-\mathbbm{1}\pbr{d^\pi}^\top)v}_\infty\\
&\leq  \sum_{k=0}^{\infty} \norm{\Phi}_\infty |\S|\norm{(\pbr{\P^\pi}^k-\mathbbm{1}\pbr{d^\pi}^\top)}_\infty\norm{v}_\infty\\
&\stackrel{\text{(b)}}{\leq}  \sum_{k=0}^{\infty} 2|\S|C_e\lambda^k\norm{v}_\infty,\\
&= \frac{2C_e|\S|}{1-\lambda}
\end{align}
where $d^\pi$ represents the stationary measure associated with the transition kernel $\P^\pi$, (a) follows from the fact that the projection matrix $\Phi$ projects vectors onto a subspace orthogonal to the subspace spanned by the all ones vector $\mathbbm{1}$ and (b) is a consequence of the irreducibility and aperiodicity assumption of the Markov chain induced under all policies. More precisely, 
for any irreducible and aperiodic stochastic matrix $A$, it is true that:
\begin{equation}
    \norm{A^n-\mathbbm{1}d^\top}_\infty \leq C_e\lambda^n,
\end{equation}
for some constants $\lambda \in [0,1), C_e<\infty $, where $d$ is stationary distribution of $A$. $\lambda$ is the coefficient of mixing and captures the rate of geometric mixing of the Markov Chain. Hence, higher the value of $\lambda,$ lower the rate of mixing. 

\item $C_p$ represents the diameter of the transition kernel as a function of the policy class and can be bound as below. 
\begin{align} 
C_p :=& \max_{u = \frac{\pi'-\pi}{\norm{\pi'-\pi}_2},\pi',\pi\in\Pi}\max_{\norm{v}_\infty\leq 1}\norm{\P^u v}_\infty,
\\=&\max_{\pi',\pi\in\Pi,\norm{v}_\infty\leq 1}\frac{\norm{ (\P^{(\pi'-\pi)}v}_\infty}{\norm{\pi'-\pi}_2}, 
\\=&\max_{\pi',\pi\in\Pi,\norm{v}_\infty\leq 1}\max_{s\in\S}\frac{\abs{ (\P^{\pi'}v)(s)- (\P^{\pi}v)(s)}}{\norm{\pi'-\pi}_2}
\\\leq&\max_{\pi',\pi\in\Pi,\norm{v}_\infty\leq 1}\max_{s\in\S}\frac{\abs{ (\P^{\pi'}v)(s)- (\P^{\pi}v)(s)}}{\norm{\pi'_s-\pi_s}_2},\qquad\text{(since  $\norm{\pi'_s-\pi_s}_2\leq \norm{\pi'-\pi}_2$)},
\\=&\max_{\pi',\pi\in\Pi,\norm{v}_\infty= 1}\max_{s\in\S}\frac{\abs{ (\pi'_s)^\top\P_sv-(\pi_s)^\top\P_sv}}{\norm{\pi'_s-\pi_s}_2},\qquad\text{(where $\P_s(a,s') = P(s'|s,a), \pi_s(a):=\pi(s,a)$)},
\\=&\max_{\pi',\pi\in\Pi,\norm{v}_\infty\leq 1}\max_{s\in\S}\frac{\abs{ (\pi'_s-\pi_s)^\top\P_sv}}{\norm{\pi'_s-\pi_s}_2},\\\leq&\max_{\pi',\pi\in\Pi,\norm{v}_\infty\leq 1}\max_{s\in\S}\frac{\norm{\pi'_s-\pi_s}_1\norm{\P_sv}_\infty}{\norm{\pi'_s-\pi_s}_2},\qquad\text{(from Holder's inequality)}\\=&\max_{\pi',\pi\in\Pi}\max_{s\in\S}\frac{\norm{\pi'_s-\pi_s}_1}{\norm{\pi'_s-\pi_s}_2},
\\\leq &\sqrt{|\A|}.
\end{align}

\item $C_r$ represents the diameter of the single step reward function as a function of the policy class and can be bound as below.
\begin{align}
C_r :=& \max_{u = \frac{\pi'-\pi}{\norm{\pi'-\pi}_2},\pi',\pi\in\Pi}\norm{r^u }_\infty
\\=&\max_{\pi',\pi\in\Pi}\frac{\norm{ r^{\pi'}-r^{\pi}}_\infty}{\norm{\pi'-\pi}_2} 
\\=&\max_{\pi',\pi\in\Pi}\max_{s\in\S}\frac{\abs{ r^{\pi'}(s)- r^{\pi}(s)}}{\norm{\pi'-\pi}_2}
\\\leq&\max_{\pi',\pi\in\Pi}\max_{s\in\S}\frac{\abs{ r^{\pi'}(s)- r^{\pi}(s)}}{\norm{\pi'_s-\pi_s}_2},\qquad\text{(as $\norm{\pi'_s-\pi_s}_2\leq \norm{\pi'-\pi}_2$)},
\\=&\max_{\pi',\pi\in\Pi}\max_{s\in\S}\frac{\abs{ (\pi'_s)^\top r_s-(\pi_s)^\top r_s}}{\norm{\pi'_s-\pi_s}_2},\qquad\text{(where $r_s(a) = r(s,a),\pi_s(a)=\pi(s,a)$)},
\\=&\max_{\pi',\pi\in\Pi}\max_{s\in\S}\frac{\abs{ (\pi'_s-\pi_s)^\top r_s}}{\norm{\pi'_s-\pi_s}_2},\\\leq&\max_{\pi',\pi\in\Pi}\max_{s\in\S}\frac{\norm{\pi'_s-\pi_s}_1\norm{r_s}_\infty}{\norm{\pi'_s-\pi_s}_2},\qquad\text{(from Holder's inequality)}\\=&\max_{\pi',\pi\in\Pi}\max_{s\in\S}\frac{\norm{\pi'_s-\pi_s}_1}{\norm{\pi'_s-\pi_s}_2},\\\leq&\sqrt{|\A|}.
\end{align}
    \end{enumerate}
Since the directional derivatives considered are all within the policy class, the analysis gives rise to constants such as $C_p$ and $C_r$, which are functions of the underlying policy class. These constants capture the MDP complexity by the virtue of their definition and are an artifact of this proof technique. 
\end{proof}

\section{Convergence of average reward projected policy gradient}

\begin{lemma}
\label{conv:proj}
    For any convex set $\X\subseteq \R^d$, any point $a\in\mathcal{X}$, and any update direction $u\in\R^d$, let
$b = \textbf{Proj}_{\mathcal{X}}(a+u)$ be the projection of $a+u$ onto $\mathcal{X}$. It is true that
\begin{enumerate}
    \item $\langle u, b-a\rangle \geq \norm{b-a}_2^2.$
    \item $\innorm{c-b,u-(b-a)}\leq 0,\qquad \forall c\in\X.$
\end{enumerate}
\end{lemma}
\begin{proof}
    The formal proof can be found in \citep{beck2014introduction,kumar2023towards}.

    However, the proof follows trivially from the geometrical representation of projection (see Figure \ref{fig:CPL}), and the fact that the hyperplane separates a convex set from a point not in the set. 
\begin{figure}[H]
\centering
        \includegraphics[scale=0.15]{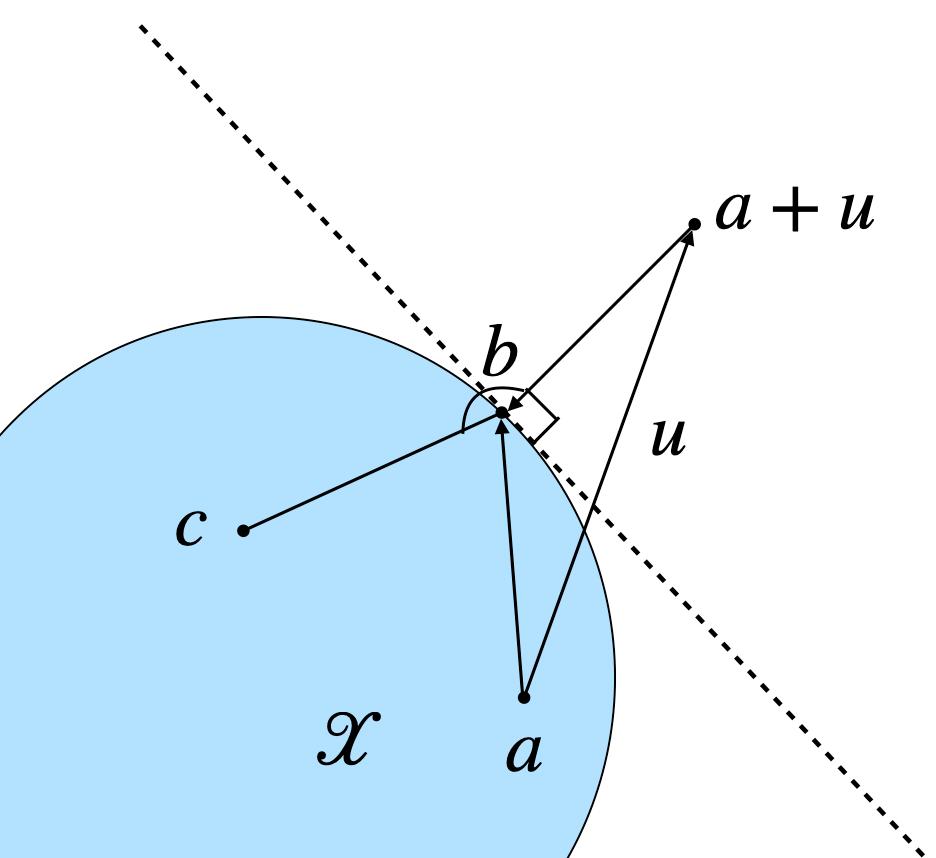}
        \caption{Convex Projection}
        \label{fig:CPL}
    \end{figure}

Intuitively, the proof of the lemma can be interpreted as below.
\begin{enumerate}
    \item Since the angle between vectors $(a-b)$ and $((a+u) - b)$ is greater than $90$ degrees, it is true that $\langle a - b, (a+u) - b\rangle \leq 0$,  which then directly implies $\norm{b-a}_2^2\leq \langle u, b-a\rangle.$
    \item The angle between vectors $(c - b)$  and $((a+u)-b)$ is greater than $90$ degrees $\forall c\in\X$, therefore $\innorm{c-b,u-(b-a)}\leq 0.$
    \end{enumerate}
\end{proof}

\subsection{Proof of Lemma 5}
\begin{lemma}\label{rs:suffIncrLemma:proof} The average reward iterates $\rho^{\pi_k}$ generated from projected policy gradient satisfy the following,
\begin{align*}
     \rho^{\pi_{k+1}} - \rho^{\pi_k}   \geq \frac{L^\Pi_2}{2}\norm*{\pi_{k+1}-\pi_k}^2, \qquad \forall k\geq 0.
\end{align*}
where $L_2^\Pi$ is the restricted smoothness constant associated with average reward $\rho^\pi.$
\end{lemma}
\begin{proof}
From the restricted smoothness of the robust return, we have \begin{align*}
        \rho^{\pi_{k+1}} &\geq \rho^{\pi_{k}} + \left\langle{ \frac{d\rho^\pi}{d\pi}\Bigg|_{\pi=\pi_k},\pi_{k+1}-\pi_k}\right\rangle - \frac{L^\Pi_2}{2}\norm{\pi_{k+1}-\pi_k}^2,\\&= \rho^{\pi_{k}} + L^\Pi_2\left\langle{ \frac{1}{L^\Pi_2}\frac{d\rho^\pi}{d\pi}\Bigg|_{\pi=\pi_k},\pi_{k+1}-\pi_k}\right\rangle - \frac{L^\Pi_2}{2}\norm{\pi_{k+1}-\pi_k}^2,\\&\geq \rho^{\pi_{k}} + L^\Pi_2\norm{\pi_{k+1}-\pi_k}^2 - \frac{L^\Pi_2}{2}\norm{\pi_{k+1}-\pi_k}^2.
\end{align*}
The last inequality follows from the projected gradient ascent policy update rule and item 1 of Lemma \ref{conv:proj}. Note that the proof only relies on the convexity of the projection set $\Pi$ and the smoothness of the objective function. 
\end{proof}

\subsection{Proof of Lemma 7}
\begin{lemma} \label{rs:GDL}
The suboptimality of a policy $\pi$ can be bounded from above as:
\begin{equation}
    \rho^{*}-\rho^\pi \leq C_{PL}\max_{\pi'\in\Pi}\left\langle\pi'-\pi,\frac{\partial\rho^{\pi}}{\partial\pi}\right\rangle,\qquad \forall \pi\in\Pi,
\end{equation}
    where $C_{PL} = \max_{\pi,s}\frac{\partial^{\pi^*}(s)}{\partial^\pi(s)}$ and $\rho^*$ is the optimal average reward.
\end{lemma}
\begin{proof}
Average Reward Performance Difference Lemma states that
\begin{align}
\rho^*-\rho^\pi =& \sum_{s\in\S}d^{\pi^*}(s)Q^{\pi}(s,a)[\pi^*(a|s)-\pi(a|s)]\\
\leq& \max_{\pi'}\sum_{s\in\S}d^{\pi^*}(s)Q^{\pi}(s,a)[\pi'(a|s)-\pi(a|s)]\\
=& \sum_{s\in\S}\frac{d^{\pi^*}(s)}{d^\pi(s)}d^\pi(s)\max_{\pi'_s}Q^{\pi}(s,a)[\pi'(a|s)-\pi(a|s)]\\
=& \sum_{s\in\S}\frac{d^{\pi^*}(s)}{d^\pi(s)}\underbrace{d^\pi(s)\max_{\pi'_s}Q^{\pi}(s,a)[\pi'(a|s)-\pi(a|s)]}_{\geq 0 , \quad \text{($=0$ when $\pi'_s = \pi_s$)}}\\
\leq & \sum_{s\in\S}\bigm(\max_{\pi,s}\frac{d^{\pi^*}(s)}{d^\pi(s)}\bigm)d^\pi(s)\max_{\pi'_s}Q^{\pi}(s,a)[\pi'(a|s)-\pi(a|s)]\\
= & C_{PL}\max_{\pi'}\sum_{s\in\S}d^\pi(s)Q^{\pi}(s,a)[\pi'(a|s)-\pi(a|s)]\\
\stackrel{\text{(a)}}{=} & C_{PL}\max_{\pi'}\innorm{\frac{d\rho^\pi}{d\pi}, \pi'-\pi}, 
\end{align}
where (a) follows from the average reward policy gradient theorem. 
\end{proof}

\subsection{Proof of Lemma 8}
\begin{lemma}
\label{lemma_B4}
Let $\pi_{k+1}$ represent the policy iterates obtained through projected policy gradient. For any policy $\pi'\in\Pi$, it is true that,
        \begin{equation}
         \Big\langle{ \frac{\partial \rho^{\pi_{k+1}}}{\partial \pi_{k+1}}, \pi'-\pi_{k+1}} \Big\rangle \leq  4\sqrt{|\S|}L^\Pi_2\norm{\pi_{k+1}-\pi_k}_2,
    \end{equation}
\end{lemma}
\begin{proof}
        For all $x,y\in C$, we have:
    \begin{align}
        \left\langle \frac{\partial \rho^{\pi_{k+1}}}{\partial\pi_{k+1}}, \pi'-\pi_{k+1}\right\rangle &=    \left\langle \frac{\partial \rho^{\pi_{k+1}}}{\partial\pi_{k+1}} - \frac{\partial\rho^{\pi_k}}{\partial \pi_k} +\frac{\partial\rho^{\pi_k}}{\partial \pi_k}, \pi'- \pi_{k+1}\right\rangle \\
        &=   \left\langle \frac{\partial \rho^{\pi_{k+1}}}{\partial\pi_{k+1}} - \frac{\partial \rho^{\pi_{k}}}{\partial\pi_{k}}, \pi'-\pi_{k+1}\right\rangle + \left\langle  \frac{\partial \rho^{\pi_{k}}}{\partial\pi_{k}} , \pi'- \pi_{k+1}\right\rangle \\
         &\leq  \left\| \frac{\partial \rho^{\pi_{k+1}}}{\partial\pi_{k+1}} - \frac{\partial \rho^{\pi_{k}}}{\partial\pi_{k}} \right\| \left\| \pi'-\pi_{k+1}\right\| + \left\langle  \frac{\partial \rho^{\pi_{k}}}{\partial\pi_{k}} , \pi'- \pi_{k+1}\right\rangle \\
         &\stackrel{\text{(a)}}{\leq} L^\Pi_2 \norm{ \pi_{k+1}-\pi_k } \norm{ \pi'-\pi_{k+1}} + \left\langle  \frac{\partial \rho^{\pi_{k}}}{\partial\pi_{k}} , \pi'- \pi_{k+1}\right\rangle, \label{eq: tbc}
    \end{align}
    where (a) uses smoothness of average reward. Thus, we may continue the chain of inequalities as
    \begin{align*}
      \eqref{eq: tbc}  \! &= L^\Pi_2 \norm{ \pi_{k+1}-\pi_k } \norm{ \pi'-\pi_{k+1}}+ \left\langle  \frac{\partial \rho^{\pi_{k}}}{\partial\pi_{k}} \!-\!L^\Pi_2(\pi_{k+1}-\pi_k), \pi'- \pi_{k+1}\right\rangle \!+\! L^\Pi_2 \left\langle\pi_{k+1}-\pi_k, \pi'- \pi_{k+1}\right\rangle \\
        &\leq   2L^\Pi_2 \norm{ \pi_{k+1}-\pi_k } \norm{ \pi'-\pi_{k+1}}+ \left\langle \frac{\partial \rho^{\pi_{k}}}{\partial\pi_{k}} -L^\Pi_2(\pi_{k+1}-\pi_k), \pi'- \pi_{k+1}\right\rangle\\      &\leq   2L^\Pi_2 \norm{ \pi_{k+1}-\pi_k } \norm{ \pi'-\pi_{k+1}}+ L^\Pi_2\underbrace{\left\langle\frac{1}{L^\Pi_2} \frac{\partial \rho^{\pi_{k}}}{\partial\pi_{k}} -(\pi_{k+1}-\pi_k), \pi'- \pi_{k+1}\right\rangle}_{\leq 0,\qquad (\text{From item 2 of Lemma \ref{conv:proj}})} \\&\leq   2L^\Pi_2 \norm{ \pi_{k+1}-\pi_k } \norm{ \pi'-\pi_{k+1}}\\
        &\leq 2L^\Pi_2 \norm{ \pi_{k+1}-\pi_k }\mathbf{diam}(\Pi). 
    \end{align*}
    The diameter of the policy class $\Pi$, can be upper bounded as 
    \begin{equation}
        \mathbf{diam}(\Pi)^2 = \max_{\pi,\pi} \sum_{s}\norm{\pi'_s-\pi_s}^2_2\leq  \max_{\pi',\pi} \sum_{s}\norm{\pi'_s-\pi_s}_1^2 \leq  4S.
    \end{equation}
This yields the result.
\end{proof}

\begin{lemma}
\label{suboptimality_recursion}
    The scaled sub-optimality  $a_k := \frac{\rho^*_\Rc -\rho_k}{32 L_2^\Pi |\S| C_{PL}^2}$ follows the recursion
    \begin{equation}
        a^2_{k+1}+a_{k+1}-a_k \leq 0. 
    \end{equation}
\end{lemma}
\begin{proof}
From Lemma \ref{rs:GDL}, we know that,
\begin{equation}
    \rho^{*}-\rho^{\pi_{k+1}} \leq C_{PL}\left\langle\pi'-{\pi_{k+1}},\frac{\partial\rho^{{\pi_{k+1}}}}{\partial{\pi_{k+1}}}\right\rangle,\qquad \forall \pi'\in\Pi,
\end{equation}
From Lemma \ref{lemma_B4}, we know that,
\begin{equation}
    \Big\langle{ \frac{\partial \rho^{\pi_{k+1}}}{\partial \pi_{k+1}}, \pi'-\pi_{k+1}} \Big\rangle \leq  4\sqrt{|\S|}L^\Pi_2\norm{\pi_{k+1}-\pi_k}_2.
\end{equation}
From Lemma \ref{rs:suffIncrLemma:proof}, we know that, 
\begin{equation}
    \norm*{\pi_{k+1}-\pi_k}_2 \leq \sqrt{\frac{2\pbr{\rho^{\pi_{k+1}} - \rho^{\pi_k}}}{L_2^\Pi}}, \qquad \forall k\geq 0.
\end{equation}
Combining the above equations yields,
\begin{equation}
    \rho^{*}-\rho^{\pi_{k+1}} \leq \sqrt{32C_{PL}^2 L_2^\Pi |\S| \pbr{\rho^{\pi_{k+1}} - \rho^{\pi_k}}}
\end{equation}
This thus yields,
\begin{equation}
   \pbr{\frac{\rho^{*}-\rho^{\pi_{k+1}}}{32C_{PL}^2 L_2^\Pi |\S|}}^2 + \pbr{\frac{\rho^{*}-\rho^{\pi_{k+1}}}{32C_{PL}^2 L_2^\Pi |\S|}} - \pbr{\frac{\rho^{*}-\rho^{\pi_{k}}}{32C_{PL}^2 L_2^\Pi |\S|}} \leq 0.
\end{equation}
\end{proof}
A more detailed interpretation of this Lemma can be found in \citep{kumar2023towards}.

\subsection{Proof of Theorem 1}
We restate the theorem for the sake of convenience. 
\begin{theorem}
Let $\rho^{\pi_k}$ represent the average reward iterates obtained through projected policy gradient. These iterates converge to the optimal average reward according to
    \begin{equation}
        \rho^* -\rho_{\pi_{k+1}} \leq \max\left(\frac{128S L_2^\Pi C_{\textsc{PL}}^2}{k},2^{-\frac{k}{2}}(\rho^* -\rho_0)\right).
    \end{equation}
\end{theorem}
\begin{proof}
    From Lemma \ref{suboptimality_recursion}, we have the following sub-optimality recursion,
    \begin{align*}
      &a^2_{k+1} + a_{k+1} - a_k \leq 0\\
    \implies & \frac{a_{k+1}}{a_k} + \frac{1}{a_k}-\frac{1}{a_{k+1}}\leq 0,\qquad\text{(dividing by $a_{k+1}a_k$)}\\
    \implies &\sum_{k=0}^{K} \Bigm[\frac{a_{k+1}}{a_k} + \frac{1}{a_k}-\frac{1}{a_{k+1}}\Bigm]\leq 0,\qquad\text{}\\
    \implies &\sum_{k=0}^{K}\frac{a_{k+1}}{a_k} + \frac{1}{a_0}-\frac{1}{a_{K+1}}\leq 0,\qquad\text{(telescoping sum)}\\
    \implies &  \frac{1}{a_{K+1}}-\frac{1}{a_0}\geq \sum_{k=0}^{K}\frac{a_{k+1}}{a_k},\qquad\text{(re-arranging)}.
    \end{align*}
Now, begin by case by case.\\
\textbf{Case 1:}
There exist $K/2$ terms such that 
$\frac{a_{k+1}}{a_k}\geq \frac{1}{2}$, that is $\{k\mid \frac{a_{k+1}}{a_k}\geq \frac{1}{2}\}$. This implies
\begin{align*}
    &\sum_{k=0}^{K}\frac{a_{k+1}}{a_k} \geq \frac{K}{4}\\
\implies &\frac{1}{a_{K+1}}-\frac{1}{a_0} \geq \frac{K}{4}\\
\implies &\frac{1}{a_{K+1}}\geq \frac{1}{a_0} + \frac{K}{4}\\
\implies &a_{K+1}\leq \frac{1}{\frac{1}{a_0} + \frac{K}{4}}\leq \frac{4}{ K}.
\end{align*}
\textbf{Case 2:}
There exist $K/2$ terms such that 
$\frac{a_{k+1}}{a_k}\leq \frac{1}{2}$, that is $\abs{\mathcal{K}}\geq \frac{K}{2}$ where $ \mathcal{K}:=\{k\mid \frac{a_{k+1}}{a_k}\leq \frac{1}{2}\}$. By definition, we have 
\begin{align*}
\frac{a_{K+1}}{a_0} &= \prod_{k=0}^{K}\frac{a_{k+1}}{a_k},\\
&= \prod_{k\in\mathcal{K}}\frac{a_{k+1}}{a_k}\prod_{k\notin\mathcal{K}}\frac{a_{k+1}}{a_k},
&\qquad \text{(by our assumption $\abs{\mathcal{K}}\geq \frac{K}{2}$)},\\
&\leq \left(\frac{1}{2}\right)^{\frac{K}{2}}\prod_{k\notin\mathcal{K}}\frac{a_{k+1}}{a_k},
&\qquad \text{(from Lemma \ref{rs:suffIncrLemma:proof} $a_{k+1}\leq a_k$)},\\
&\leq \left(\frac{1}{2}\right)^{\frac{K}{2}}.
\end{align*}
From both cases, it can be concluded that either
$a_{k+1} \leq \frac{4}{K} $ or $a_{k+1} \leq a_02^{-\frac{K}{2}}$. Hence, we obtain
   \[ a_{k+1} \leq \max\left\{\frac{4}{k},a_02^{-\frac{K}{2}}\right\}.\]
Appropriately substituting for $a_k$, we get the desired result. Note that the constants $\frac{4}{K}, 2^{-\frac{K}{2}}$ can be further optimized by appropriately choosing constants such as $\frac{1}{2},\frac{K}{2}$, in the proof.  The solution to the above sub-optimality recursion is taken from \citep{PGConvRate,kumar2023towards}, and presented here for the sake of completeness.
\end{proof}

\section{Extension to Discounted Reward MDPs}
\label{appendix_C}
Our techniques extends to discounted reward MDPs which has state-of-the-art iteration complexity of $O(\frac{SA}{(1-\gamma)^5\epsilon})$ \citep{PGConvRate} but showcases no dependence on the hardness of the MDP. Our approach improves on this bound, yielding $O(\frac{SL^\Pi_2}{\epsilon})$ iteration complexity, where $L^\Pi_2=C^2_p\widehat{C}^2_m\kappa_r+C_p\widehat{C}_mC_r
+ (C_p+1)(\widehat{C}^2_mC_p\kappa_r+\widehat{C}_mC_r)+4(\widehat{C}^3_mC^2_p\kappa_r+\widehat{C}^2_mC_pC_r)$,where
$\widehat{C}_m := \norm{(I-\gamma \P^\pi)^{-1}}$. It is straightforward to see that $\widehat{C}_m \leq \frac{1}{1-\gamma}$.
Hence, the iteration complexity improves to an $O(\frac{L^\Pi_2S}{(1-\gamma)^5\epsilon})$, since constants such as $\kappa_r \leq 2$ and $C_p, C_r \leq \sqrt{|\A|}$. Further, the approach considered in this paper provides faster convergence rates for MDPs with low complexity, i.e., MDPs that have low values of $C_p$ or $C_r$.
The exact performance bounds can be obtained from an approach similar to the one outlined in \citep{kumar2023towards}, where $L^\Pi_2$ represents the restricted smoothness constant of the discounted return $\rho^\pi_\gamma$. This constant can be derived through a process analogous to the one described in this paper.

For instance, consider a trivial MDP for which $C_p =0$ or $\kappa_r=0$ (implies $C_r=0$), i.e., an MDP where the transition kernel is independent of the action enacted. For this trivial MDP every policy is an optimal policy. The state of the art convergence guarantees \citep{PGConvRate}, still requires $O(SA\epsilon^{-1})$ iterations for $\epsilon$ close optimal policy. Whereas, the performance bounds presented in this paper predict $O(1\epsilon^{-1})$ iterations for convergence.

\section{Simulation Details}
\label{appendix_D}
We consider MDPs of size 20, i.e., $\pbr{|\S|,|\A|}=\cbr{20,20}$. The rewards corresponding to these transition kernels are randomly generated and are held constant. Three sets of transition kernels are randomly generated, each corresponding to a different value of $C_p$. We run the policy gradient algorithm considered in this paper for each MDP setting and plot the overall change in average reward as a function of iterations. 
\begin{figure}[H]
    \centering
    \includegraphics[scale=0.55]{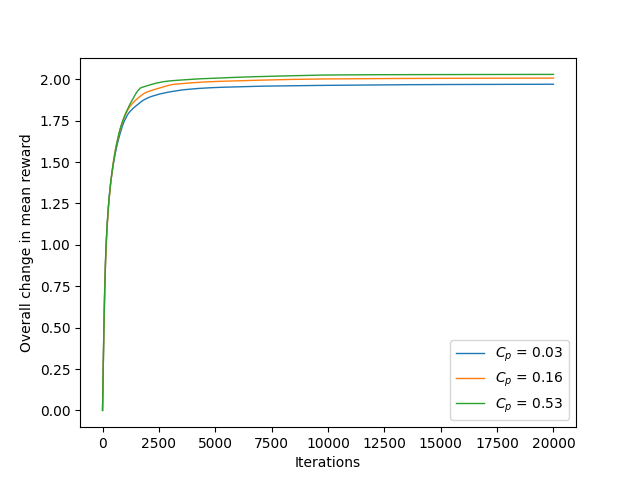}
    \caption{Convergence as a function of $C_p$}
    \label{figure:conv_Cp}
\end{figure}
Figure \ref{figure:conv_Cp} indicates that policy gradient in MDPs corresponding to lower values of $C_p$ converges relatively faster than for MDPs corresponding to higher values of $C_p$. Hence, the performance bounds obtained in Theorem 1 are in some sense, more representative of the empirical convergence trend of the policy gradient algorithm.

In Table \ref{tab:append_my_label}, we notice that the upper bound for $C_p$ and $C_r$ are of order $\sqrt{|\A|}$. Here, we resent additional details regarding the calculation of $C_r$ and $C_p$. Recall their definitions from Table \ref{tab:append_my_label}. Since there seems to be no closed formula for $C_r, C_p$, their value cannot be well approximated by using simple black-box optimizer. Thus, we perform a Bayesian search across our policy space in order to find $\pi, \pi'$ that maximize the expressions under consideration. Furthermore, in order to reduce our search space, we limit ourselves to policies which are constant across states, meaning $\pi_s(a) = \pi_{s'}(a)$ for all states $s,s'$. These policies can be utilized to realize higher values for $C_p, C_r$. Note that while using such searching methods does not guarantee us convergence to the maximal value, it does allow us to obtain a good approximation which allows us to compare the scaling of these coefficients with $(\S,\A)$.
It is possible to construct specific MDPs, where $C_p$ and $C_r$ can realize a value of $\sqrt{|\A|}.$ However, these specific MDPs rarely occur in practice. In this section, we examine the scaling of $C_r, C_p$ when compared to $\sqrt{A}$ across varying state and action space cardinalities. We consider some practical settings such as deterministic MDPs and sparse rewards. 

Figure \ref{variation_Cp} represents the variation of $C_p$ as a function of state and action spaces. Specifically, a comparison is made between two types of probability kernels: one generated from a uniform distribution and another generated by random permutations of the identity matrix (ensuring the MDP remains irreducible). For each instance, $C_p$ is calculated for five randomly generated MDPs. It is observed that for both deterministic and sparse MDPs, there appears to be no scaling of $C_p$ with $\sqrt{A}$.

\begin{figure}[H]
    \centering
    \subfigure[$C_p$ as a function of uniform kernels]{
    \includegraphics[width=0.48\textwidth]{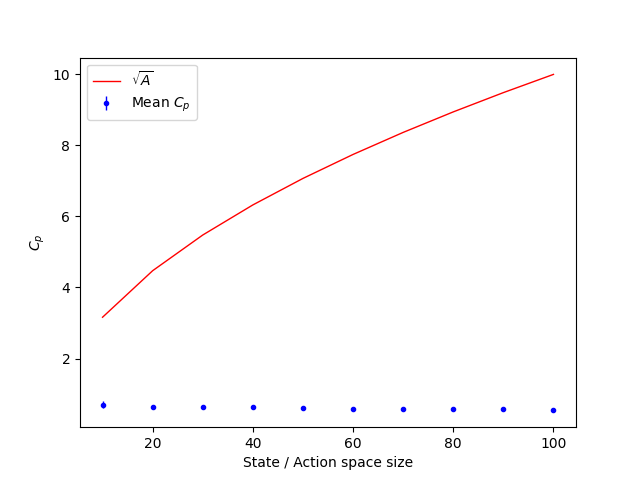}
    \label{fig1}
    }
    \hfill
    \subfigure[$C_p$ as a function of sparse kernels]{
    \includegraphics[width=0.48\textwidth]{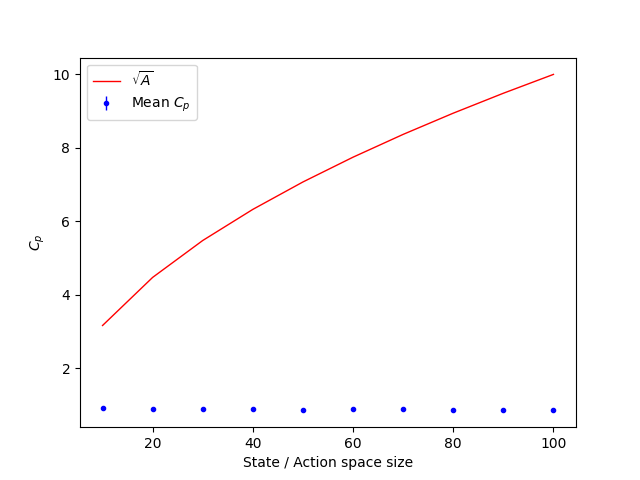}
    \label{fig2}
    }
    \caption{Variation of $C_p$ with state and action space cardinality.}
    \label{variation_Cp}
\end{figure}

An additional experiment is conducted to analyze randomly generated and sparse rewards, investigating the scaling behavior of $C_r$ with respect to $A$ in this context. The results of this experiment are depicted in Figure \ref{fig:exp_CR}. It is observed that in both scenarios, $C_r$ does not exhibit scaling with $A$.

Similar to $C_p$, the upper bound on $C_r$ is tight and can be achieved when the reward matrix is dense with just two possible values. However, real-world reward matrices are often random or sparse. From both experiments, it is inferred that in many practical settings, such as MDPs with sparse rewards and deterministic transitions, the scaling with $A$ is less pronounced than in the worst-case scenario. The code for these experiments can be accessed at \url{https://anonymous.4open.science/r/avpg_convergence-3F03}.

\begin{figure}[H]
    \centering
    \subfigure[Convergence rate versus state space cardinality]{
    \includegraphics[width=0.48\textwidth]{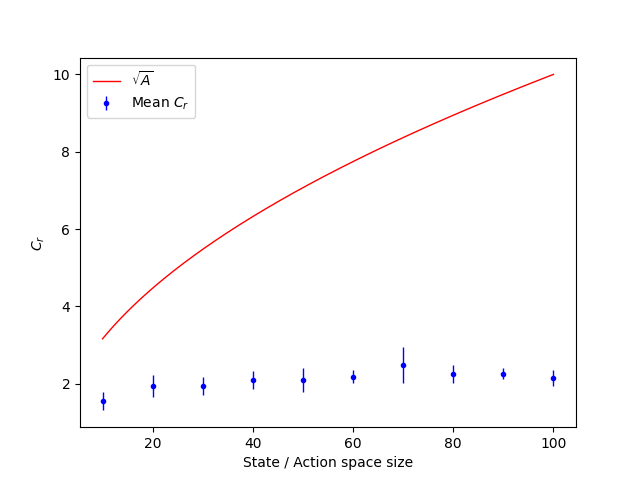}
    \label{fig1e}
    }
    \hfill
    \subfigure[Convergence rate versus diameter of reward]{
    \includegraphics[width=0.48\textwidth]{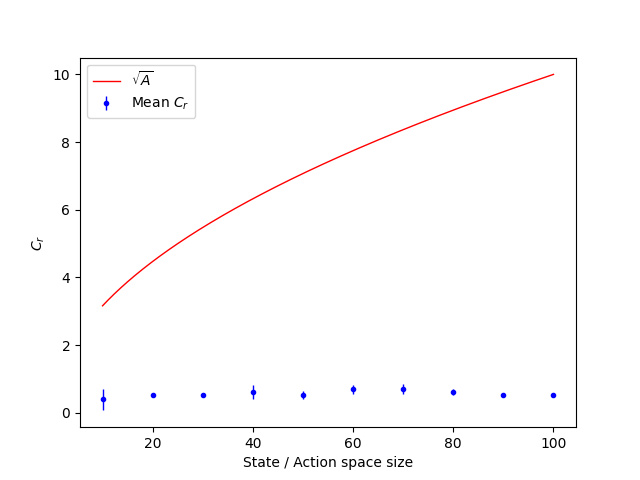}
    \label{fig2f}
    }
    \caption{Improvement in average reward as a function of MDP complexity}
    \label{fig:exp_CR}
\end{figure}

\end{document}